\documentclass{article}
\PassOptionsToPackage{numbers}{natbib}
\usepackage[preprint]{neurips_2022}
\usepackage{color}
\usepackage{amsthm}
\usepackage{amsmath}
\usepackage{amsfonts}
\usepackage{amssymb}
\usepackage[dvipsnames]{xcolor}
\usepackage{algorithmic}
\usepackage[linesnumbered,ruled]{algorithm2e}
\usepackage{comment}
\usepackage{thmtools,thm-restate}
\usepackage{multirow}
\usepackage{subcaption}
\usepackage{graphicx}
\usepackage{gensymb}
\usepackage{enumitem}
\usepackage{bbm}
\usepackage{wrapfig}

\newtheorem{theorem}{Theorem}
\newtheorem{lemma}[theorem]{Lemma}

\newtheorem{claim}{Claim}

\newtheorem{observation}{Observation}
\newtheorem{definition}{Definition}

\newcommand{\alg}{\mathtt{ALG}}
\newcommand{\algprime}{\mathtt{ALG'}}
\newcommand{\algtwo}{\mathtt{\overline{ALG}}}
\newcommand{\opt}{\mathtt{OPT}}
\newcommand{\rr}{\mathtt{RR}}
\newcommand{\rew}{\mathtt{Rew}}
\newcommand{\slopeOC}{m_{\scriptscriptstyle OE}}
\newcommand{\slopeIO}{m_{i0}}
\newcommand{\IMAB}{IMAB }
\newcommand{\imab}{IMAB}
\newcommand{\compRatio}{\mathtt{CR}}
\newcommand{\instance}{\langle k, (f_i)_{i\in [k]}\rangle}
\newcommand{\Ar}{\mathsf{Ar}}

\ifdefined\DEBUG
	
	\newcommand{\vis}[1]{\textcolor{OliveGreen}{#1}}
	\newcommand{\vin}[1]{\textcolor{NavyBlue}{#1}}
	\newcommand{\gan}[1]{\textcolor{purple}{#1}}
	\def\rem#1{{\marginpar{\raggedright\scriptsize #1}}}
	\newcommand{\arir}[1]{\rem{\textcolor{Red}{$\bullet$ #1}}}
	\newcommand{\vish}[1]{\rem{\textcolor{OliveGreen}{$\bullet$ #1}}}
	\newcommand{\vine}[1]{\rem{\textcolor{NavyBlue}{$\bullet$ #1}}}
	\newcommand{\gane}[1]{\rem{\textcolor{purple}{$\bullet$ #1}}}
\else
	
	\newcommand{\vis}[1]{#1}
	\newcommand{\vin}[1]{#1}
	\newcommand{\gan}[1]{#1}
	
	\newcommand{\arir}[1]{}
	\newcommand{\vish}[1]{}
	\newcommand{\vine}[1]{}
	\newcommand{\gane}[1]{}
\fi

\usepackage[utf8]{inputenc} % allow utf-8 input
\usepackage[T1]{fontenc}    % use 8-bit T1 fonts
\usepackage{hyperref}       % hyperlinks
\usepackage{ulem}
\usepackage{xurl}            % simple URL typesetting
\usepackage{booktabs}       % professional-quality tables
\usepackage{amsfonts}       % blackboard math symbols
\usepackage{nicefrac}       % compact symbols for 1/2, etc.
\usepackage{microtype}      % microtypography
\usepackage{xcolor}         % colors

\title{Mitigating Disparity while Maximizing Reward: \\Tight Anytime Guarantee for Improving Bandits}
\author{%
    Vishakha Patil \thanks{Equal Contribution, random order within}\\
    Indian Institute of Science, Bangalore\\
    \texttt{patilv@iisc.ac.in}\\
    \And
    Vineet Nair $^*$ \\
    Arithmic Labs\\
    \texttt{vineet@arithmic.com}\\
    \And
    Ganesh Ghalme\\
    Indian Institute of Technology, Hyderabad\\
    \texttt{ganeshghalme@ai.iith.ac.in}
    \And
    Arindam Khan\\
    Indian Institute of Science, Bangalore\\
    \texttt{arindamkhan@iisc.ac.in}
}

\begin{document}

\maketitle

\begin{abstract}
We study the Improving Multi-Armed Bandit (\imab) problem, where the reward obtained from an arm increases with the number of pulls it receives.
\vis{This model provides an elegant abstraction for many real-world problems in domains such as education and employment, where decisions about the distribution of opportunities can affect the future capabilities of communities and the disparity between them.}
A decision-maker in such settings must consider the impact of her decisions on future rewards in addition to the standard objective of maximizing her cumulative reward at any time.
In many of these applications, the time horizon is unknown to the decision-maker beforehand, which motivates the study of the \IMAB problem in the technically more challenging horizon-unaware setting.
We study the tension that arises between two seemingly conflicting objectives in the horizon-unaware setting: 
a) maximizing the cumulative reward at any time based on current rewards of the arms, and b) ensuring that arms with better long-term rewards get sufficient opportunities even if they initially have low rewards.
We show that, surprisingly, the two objectives are aligned with each other in this setting.
Our main contribution is an {\em anytime} algorithm for the \IMAB problem that achieves the \emph{best possible cumulative reward} while ensuring that the \emph{arms reach their true potential} given sufficient time.
%, thereby mitigating the existing disparities. 
\vis{Our algorithm mitigates the initial disparity due to lack of opportunity and continues pulling an arm till it stops improving.}
We prove the optimality of our algorithm by showing that a) any algorithm for the \IMAB problem, no matter how utilitarian, must suffer $\Omega(T)$ policy regret and $\Omega(k)$ competitive ratio with respect to the optimal offline policy, and b) the competitive ratio of our algorithm is $O(k)$.  
\end{abstract}

\section{Introduction}
\label{sec: introduction}
Machine Learning (ML) algorithms are increasingly being used to make or assist critical decisions that affect people in areas such as education \citep{MARCINKOWSKI2020}, employment \cite{SANCHEZ2020}, and loan lending \cite{DOBBIE2021}.
In these domains, the decisions concerning the distribution of opportunities can affect the future capabilities of individuals or communities that are impacted by these decisions. Further, these decisions may exacerbate or mitigate the disparity that exists between their capabilities.
However, most of the existing literature focuses on static settings without considering the long-term impact of algorithmic decisions \citep{CHOULDECHOVA2020,DWORK2012,HARDT2016}.
Recent work has highlighted the need to study the impact of algorithmic decisions made over multiple time steps \cite{HEIDARI2019,LIU2018,LINDNER2021}.
\vis{In this work, we model such scenarios as a variant of the multi-armed bandit problem, called Improving Multi-armed Bandits, and study the long-term impact of algorithmic decisions when the arms of the bandit \emph{evolve} over time.}

Multi-Armed Bandits (MAB) is a classic framework used to capture decision-making over multiple time steps.
We study a variant of the \vis{non-stationary} MAB problem \cite{GARIVIER2011,AUER2019}, called the Improving Multi-Armed Bandit (\imab) problem, which models scenarios where the capabilities of individuals or communities can improve based on the opportunities they receive. 
%The goal of the decision-maker is to maximize the cumulative reward at any time.
In the \IMAB problem, the decision-maker has access to $k$ arms.
Each arm has a reward function associated with it, which is unknown to the decision-maker beforehand.
At each time step, the decision-maker pulls an arm and receives a reward.
The reward obtained by pulling an arm increases with the number of pulls it receives.
The goal of the decision-maker is to pull arms in a manner that maximizes her cumulative reward.
In our examples, 
%the arms correspond to the individuals or communities, the rewards received on pulling arms to their capabilities, and the number of arm pulls to the opportunities that they have received.
the individuals or communities  correspond to the arms of the bandit, their capabilities to the rewards received on pulling the arms, and opportunities to the number of pulls that the arms receive.

The \IMAB problem has been previously studied in the horizon-aware setting \vis{with asymptotic regret guarantees} \cite{HEIDARI2016,LINDNER2021} \vis{(see Section \ref{subsec: related work})}.
However, in many practical applications, the time horizon is not known to the algorithm beforehand.
We initiate the study of the \IMAB problem in the horizon-unaware (or {\em anytime}) setting.
In contrast to horizon-aware algorithms, anytime algorithms do not know the time horizon beforehand and hence cannot tailor their decisions to the given time horizon.
Thus, an anytime algorithm must perform well for any finite time horizon without having prior knowledge of it, which poses interesting technical challenges.
Due to its theoretical and practical significance, the design of anytime algorithms for variants of the MAB problem has been of prime interest to the MAB research community (e.g., the popular UCB1 algorithm for stochastic MAB \cite{AUER2002}).

The \IMAB model is also well-motivated in the domain of algorithmic fairness. Fairness through awareness \citep{DWORK2012} is a well-accepted notion of fairness that requires that \emph{similar} individuals or communities be treated similarly.
In the \IMAB model, this could mean quantifying similarity \vis{(or equivalently, disparity)} based on the current rewards (capability) of the arms.
%Given that the individuals or communities correspond to arms of the bandit, one way of quantifying \emph{similarity} in the \IMAB model could be based on the difference between the instantaneous rewards (current abilities) of pairs of arms.
However, we argue that such a blind comparison may be fallacious.
For instance, historical marginalization could lead to differences in the abilities of different individuals or communities to perform a given task, for example, the racial gap observed in SAT scores \cite{REEVES2017}.
%\footnote{\url{https://www.brookings.edu/research/race-gaps-in-sat-scores-highlight-inequality-and-hinder-upward-mobility/}}
One way of mitigating such differences that has also been studied in the MAB literature \cite{LI2019,PATIL2020,CHEN2020}, is through \emph{affirmative action}, where the decision-maker allocates some opportunities to individuals based on attributes such as their race, gender, caste, etc.
In some parts of the world, such policies have been in place for decades (see reservation system in India \cite{SAHOO2009}), while they are banned in several US states \cite{BAKER2019}.
%\footnote{\url{https://www.brookings.edu/blog/brown-center-chalkboard/2019/04/12/why-might-states-ban-affirmative-action/}}
Another popular notion of fairness in the MAB literature is meritocratic fairness \citep{JOSEPH2016}, where arms are compared solely based on their current rewards.
However, in the \IMAB model, meritocracy would identify individuals that are gifted early and provide them more opportunities which would suppress the growth of \emph{late bloomers}, i.e., the individuals that would go on to perform well had they been given more opportunities.
This detrimental effect of meritocracy has also been observed in the real world; for example, the education system in Singapore \cite{GLOBAL2018}.%\footnote{\url{https://lkyspp.nus.edu.sg/gia/article/meritocracy-in-singapore-solution-or-problem}}

\begin{wrapfigure}{r}{0.4\textwidth}
  \begin{center}
    \includegraphics[width=0.32\textwidth]{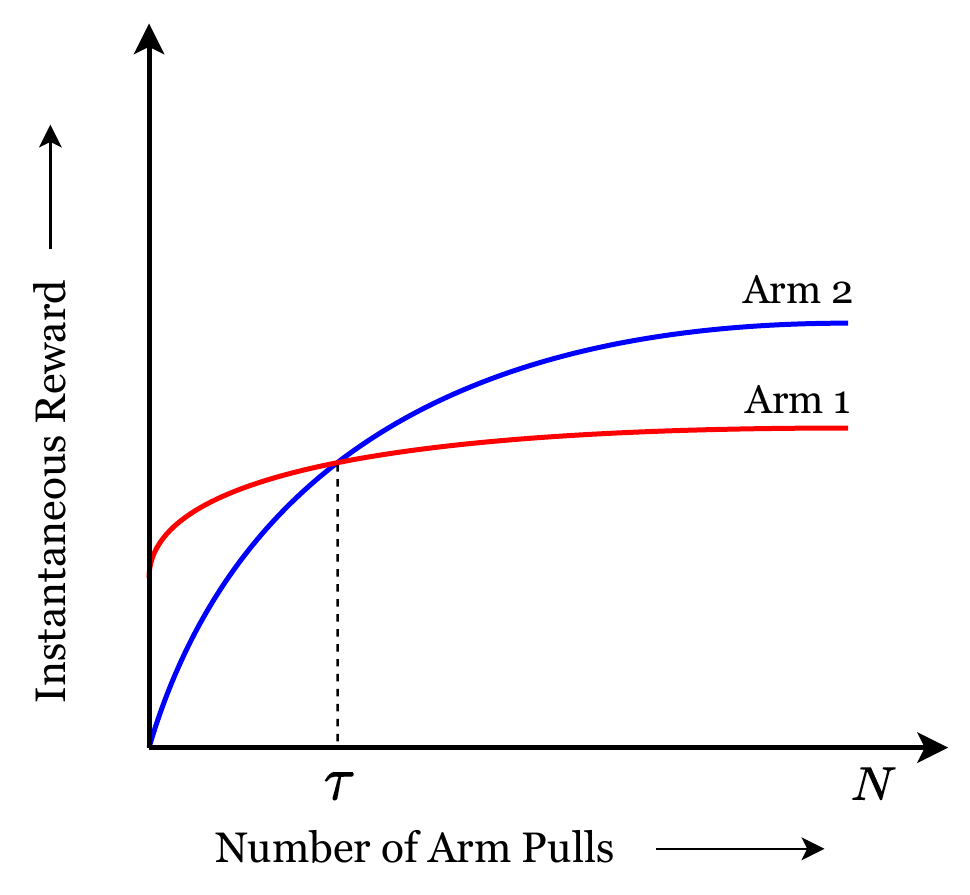}
  \end{center}
  %\label{fig: intro example}
  \caption{Two-armed \IMAB Instance}
  \label{fig: intro example}
\end{wrapfigure}
We give a simple example to demonstrate the challenges in the anytime \IMAB problem.
Figure \ref{fig: intro example} shows a two-armed \IMAB instance.
%, where the x-axis is the number of arm pulls and the y-axis is the value of the corresponding reward.
We emphasize here that the rewards of arms change only when an arm is pulled.
In particular, the x-axis denotes the number of arm pulls of an arm and not the time horizon.
As the figure shows, arm 1 is an early gifted arm, and arm 2 is a late bloomer.
Here, a myopic decision-maker that pulls arms based only on the instantaneous rewards will almost never pull arm $2$.
An algorithm that majorly plays arm $1$ may obtain good cumulative reward for a horizon that is less than say $\tau$.
However, for a horizon slightly larger than $\tau$ this algorithm could perform poorly.
Additionally, an algorithm that mostly plays arm 1 will increase the disparity between the two arms.
This highlights the challenges faced by an anytime algorithm in balancing exploitation (pulling arm 1) with exploration (enabling arm 2) in a way that gives good cumulative reward.

\textbf{Our Results: }Our contributions in this paper are twofold. First, we contribute strong theoretical results to the long line of literature on non-stationary bandits, in particular, rested bandits, where the rewards obtained from an arm can change when pulled \cite{TEKIN2012,LEVINE2017}.
Second, we make an important conceptual and technical contribution to the study of fairness in the \IMAB problem.
We study the \IMAB problem in the horizon-unaware setting with the following objective: how does a decision-maker maximize her reward while ensuring the participants (arms) are provided with sufficient opportunities to improve and reach their true potential?
%each arm is associated with a reward function $f_i(\cdot)$ and arm $i$ gives a deterministic reward $f_i(N)$ when it is pulled for $N$-th time.
Inspired by the motivating examples and numerous studies on human learning \cite{ANDERSON1991,SON2006}, the reward functions \vis{are assumed to be bounded, monotonically increasing, and having decreasing marginal returns (diminishing returns).}
%corresponding to the arms are assumed to satisfy the properties of \vis{monotonicity} and decreasing marginal returns (see Section \ref{sec: model} for the details).
%concavity, boundedness,

Our first result shows that any algorithm, how much ever utilitarian, for the \IMAB problem suffers $\Omega(T)$ regret and has competitive ratio $\Omega(k)$ (Theorem \ref{thm: regret lower bound}).
Our main contribution \vin{is an efficient} anytime algorithm (Algorithm \ref{algo: horizon unaware improving bandits}) which has a competitive ratio of $O(k)$ for the \IMAB problem in the horizon-unaware case (Theorem \ref{theorem: k competitive algo proof}).\footnote{Informally, the competitive ratio of an algorithm is the worst-case ratio of its reward to that of the offline optimal (see Definition \ref{def: competitive ratio}).}
\vis{An interesting and important property of our proposed algorithm is that it continues pulling an arm until it reaches its true potential (Theorem \ref{theorem: fairness of optimal algo}), thus mitigating the disparity that existed due to lack of opportunity.}
%An interesting and important outcome of our proposed algorithm is that it ensures that each arm achieves its true potential for large enough time horizons (Theorem \ref{theorem: fairness of optimal algo}).
%\vis{thus mitigating the initial disparity caused due to lack of opportunity and continues pulling an arm till it stops improving.}
%\vin{In the context of fairness, this ensures that the initial disparities between the arms that existed owing to lack of opportunities are mitigated given sufficient time. }
We note that this is not accomplished by imposing any fairness constraints but establishing that it is in the best interest of the decision-maker to enable arms to achieve its true potential.
%We show that the competitive ratio of our algorithm is tight up to constants by showing that any algorithm for the \IMAB problem has a 
%competitive ratio of at least $\Omega(k)$, i.e., the lower bound on the competitive ratio is $\Omega(k)$ (Theorem \ref{thm: regret lower bound}).
\vin{The proofs of Theorems \ref{theorem: k competitive algo proof} and \ref{theorem: fairness of optimal algo} require novel techniques and intricate analysis (see Sections \ref{subsec: proof sketch of cr} and \ref{subsec: proof sketch fairness}).
%The analysis of the competitive ratio of our proposed algorithm and how it enables arms to achieve true potential, is highly non-trivial and rests crucially on several important and non-trivial properties that we show it satisfies 
The analysis of the performance of our algorithm rests crucially on several important and non-trivial properties that we show it satisfies (e.g., see Lemmas \ref{lemms: first arm to cross N pulls}, \ref{lemma: rew N by opt T for arm i},  and \ref{lemma: fairness of optimal algo}) and are the key technical contributions of our paper. }
%(e.g., see Lemmas \ref{lemms: first arm to cross N pulls}, \ref{lemma: rew N by opt T for arm i},  and \ref{lemma: fairness of optimal algo}) 
We also analyse the performance of the round-robin ($\rr$) algorithm. 
We show that while $\rr$ gives equal opportunity to all arms its competitive ratio is $\Theta(k^2)$ and hence, is sub-optimal for the decision maker (Theorem \ref{thm: round robin performance}). 

\subsection{Related Work}
\label{subsec: related work}
The \IMAB problem was introduced by \citet{HEIDARI2016}, who study the horizon-aware \IMAB problem.
Their work differs from ours in two key aspects.
First, they study the horizon-aware setting.
In particular, their algorithm uses the knowledge of $T$ at every time step. 
In contrast, our algorithm does not have prior knowledge of $T$ and must work well for any stopping time T.
%\sout{Our regret lower bound (see Section \ref{sec: lower bound}) shows that the anytime setting is harder for the \IMAB problem.}
Second, they provide an \emph{asymptotically} sub-linear regret bound in terms of instance-dependent parameters.\footnote{Asymptotically sub-linear regret bound implies that as the time horizon tends to infinity the ratio of the regret of algorithm to time horizon is zero.}
On the other hand, our results hold for \emph{any finite time horizon} $T$ and not just asymptotically.
We note here that our results are with respect to two standard performance metrics in the MAB literature, i.e., policy regret \cite{ARORA2012,HEIDARI2016,LINDNER2021} and competitive ratio \cite{IMMORLICA2019,KESSELHEIM2020,ANDREW2013,BASU2021,DANIELY2019} (see Section \ref{sec: model} for the definitions).

The area of fairness in ML has received tremendous attention in recent years \cite{BAROCAS2019}. 
However, much of this attention has been focused on fairness in static and one-shot settings such as classification \cite{KLEINBERG2016,HARDT2016,DWORK2012}. 
Recent work has also started studying the fairness aspects in models that capture sequential decision-making such as MABs \cite{JOSEPH2016,PATIL2020,LI2019,WEN2021,HOSSAIN2021} and Markov Decision Processes (MDPs) \cite{JABBARI2017,WANG2021,GHALME2022}. 
However, these works do not consider the impact of the decisions on the population on which they operate.
With a motivation similar to ours, \citet{LINDNER2021} recently studied a problem called the single-peaked bandit model, where the rewards of the arms are first non-decreasing but can then start decreasing after a point.
This class of reward functions subsumes the class of reward functions considered in \cite{HEIDARI2016}.
The results in \cite{LINDNER2021}, which are again for the horizon-aware case only, match the results in \cite{HEIDARI2016} for the class of 
reward functions considered in \imab.
%non-decreasing reward functions.

\section{Model and Preliminaries}
\label{sec: model}
Throughout we use $\mathbb{R}$ and $\mathbb{N}$ to denote the set of real and natural numbers, respectively, and $[k]$ to denote the set $\{1, 2, \ldots, k\}$ for $k\in\mathbb{N}$.

\paragraph{Model and Problem Definition:} The \IMAB model studied in our work was introduced in \cite{HEIDARI2016}.
Formally, an instance $I$ of the \IMAB problem is defined by a tuple $\instance$ where $k$ is the number of arms.
Each arm $i\in[k]$ is associated with a fixed underlying reward function denoted by $f_i(\cdot)$.
When the decision-maker pulls arm $i$ for the $n$-th time, it obtains an instantaneous reward $f_i(n)$.
%For each arm $i\in[k]$, $f_i$ denotes the reward function of arm $i$ and $f_i(n)$ denotes the instantaneous reward received by the algorithm when arm $i$ is pulled for the $n-th$ time. 
Further, $\rew_i(N)$ denotes the cumulative reward obtained from arm $i$ after it has been pulled $N$ times, i.e., $\rew_i(N) = f_i(1) + f_i(2) + \ldots + f_i(N)$.
We assume that the reward functions $f_i$, $i\in[k]$ are bounded in $[0,1]$, i.e., $f_i: \mathbb{N} \rightarrow [0,1]$.
In our motivating examples, the reward functions $f_i$ correspond to the ability of individuals to improve with more opportunity.
In the \IMAB model, $f_i$'s are assumed to be \vis{monotonically increasing} with decreasing marginal returns (aka diminishing returns).\footnote{We can think of $f_i$'s as being continuous functions, in which case monotonically increasing and decreasing marginal returns imply concavity.}
%concavity implies the decreasing marginal returns property.}
This assumption about the progression of human abilities is well-supported by literature in areas such as cognitive sciences \cite{SON2006} and microeconomics \cite{JOVANOVIC1995}.
The decreasing marginal returns property for $f_i$ states that, for all $i\in [k]$ 
\[
f_i(n+1) - f_i(n) \leq f_i(n) - f_i(n-1) \hspace{0.5cm} \text{for all } n\geq 1.
\]
Next, let $a_i$ denote the asymptote of $f_i(\cdot)$, i.e., $a_i = \lim_{n\rightarrow\infty} f_i(n).$
Since $f_i(\cdot)$ is monotonically increasing and bounded, this asymptote exists and is finite.
In the context of our motivating examples, we refer to $a_i$ as the true potential of the corresponding individual or community.
That is, how well they can perform a task given enough opportunity.
%\vine{Do we say about true potential or asymptotes here?}

Let $\alg$ be a deterministic algorithm for the \IMAB problem and $T$ be the time horizon that is unknown to $\alg$.
Let $i_t \in [k]$ denote the arm pulled by $\alg$ at time step $t \in [T]$.
%Then, $\alg$ can be written as a sequence of mappings $(\alg_1,\alg_2,\ldots,\alg_T)$ from the history to the set of arms, i.e., $$\alg_t : [k]^{t-1} \times [0,1]^{t-1} \rightarrow [k].$$\vine{Take out the definition of $\alg$ if not used anywhere!}
We use $N_i(t)$ to denote the number of pulls of arm $i$ made by $\alg$ until (not including) time step $t$, and $\alg(I,T)$ to denote the cumulative reward of $\alg$ \vis{on instance $I$} at the end of $T$ time steps.
Then,
$\alg(I, T) = \sum_{t=1}^T f_{i_t}(N_{i_t}(t)+1).$
We further note that the cumulative reward of $\alg$ after $T$ time steps only depends on the number of arm pulls of each arm and not on the order of arm pulls.
Hence, we can write $\alg(I, T) = \sum_{i\in[k]} \rew_i(N_i(T+1))$.
For brevity, we write $N_i(T+1)$ as $N_i$.
Hence, $\alg(I, T) = \sum_{i\in[k]} \rew_i(N_i)$.
\vis{When $I$ is clear from context, we use $\alg(T)$ to denote $\alg(I,T)$.}

\paragraph{Offline Optimal Algorithm for \imab:} Let $I = \instance$ be an \IMAB instance.
\vis{We use $\opt(I, T)$ to denote the offline algorithm maximizing the cumulative reward   for instance $I$ and horizon $T$.
Here, offline means that $\opt(I, T)$ knows the \IMAB instance $I$ and the time horizon $T$ beforehand.
%That is, $\opt(T)$ is the algorithm that maximizes the cumulative reward when the MAB instance \vin{and the time horizon} are known to it.
With slight abuse of notation, we also denote the cumulative reward of this algorithm by $\opt(I, T)$.
When $I$ is clear from context, we use $\opt(T)$ instead of $\opt(I, T)$. 
The following proposition shows that for the \IMAB problem, $\opt(I, T)$ corresponds to pulling a single arm for all the $T$ rounds.}
\vis{We give an alternate proof of this proposition in Appendix \ref{app: proof of offline optimal strategy}.}
%\vin{For completeness, we prove this in Proposition \ref{prop: single arm optimal policy} (proof in Appendix \ref{app: proof of offline optimal strategy}).}
%
%\begin{restatable}{proposition}[PropOfflineOptimal]
%\label{prop: single arm optimal policy}
%Suppose $I = \instance$ is an instance of the \IMAB problem and $T$ is the time horizon. 
%Then there exists an arm $j^*_T$ such that the optimal offline algorithm consists of pulling arm $j^*_T$ for $T$ time steps.
%\end{restatable}

\begin{restatable}{proposition}{PropOne}[\cite{HEIDARI2016}]
\label{prop: single arm optimal policy}
Suppose $I = \instance$ is an instance of the \IMAB problem and $T$ is the time horizon. 
Then there exists an arm $j^*_T$ such that the optimal offline algorithm consists of pulling arm $j^*_T$ for $T$ time steps.
\end{restatable}

We emphasize that $j^*_T$ in Proposition \ref{prop: single arm optimal policy} depends on the time horizon $T$, and may be different for different values of $T$.
%It is important to note that the single best arm may be different for different time horizons, i.e., the optimal policy can be different depending on the time horizon. 
\vin{We compare the performance of an online algorithm at any time $T$ with $\opt(T)$ using the performance metrics defined next}.
%We provide an alternate proof to this in Appendix \ref{app: proof of offline optimal strategy}.
%Recalling our motivating example, we can infer that the strategy that gives maximum returns to the decision-maker entails allocating the entire budget
%$B$ to same neighbourhood every fiscal year. 
%However, this is highly unfair to other neighbourhoods and may lead to further amplification of the already existing disparity among different communities.
%However, it should be clear to the reader that this is highly unfair to other neighbourhoods and may lead to further amplification of the already existing disparity among different communities.
%Further, from a practical perspective, a decision-maker is unlikely to have knowledge of $T$ and much less of $f_i$ beforehand.
%This leads us on a quest to design an algorithm with \emph{good} performance guarantees in the absence of such knowledge.

\paragraph{Performance Metrics:} In this work, our objective is to minimize the stronger notion of regret, viz. policy regret, as opposed to external regret, which is another commonly studied objective in the MAB literature \cite{ARORA2012}.
%Policy regret captures the regret that an online algorithm would suffer with respect to the optimal sequence of arm pulls that the algorithm could have made against an adaptive adversary.
%In the \IMAB model, the rewards of an arm pull can be thought of as being generated by an adaptive adversary while respecting the reward function properties we have described above.
We refer the reader to Example 1 in \cite{HEIDARI2016} for insight into how the two regret notions differ in the \IMAB model. We also provide this example in Appendix \ref{app: example from heidari} for completeness.
We next define the policy regret in the case of \IMAB problem. Henceforth, we use regret to mean policy regret unless stated otherwise.
%\todo{Justify Policy Regret}
\begin{definition}
\label{def: policy regret}
Let $\mathcal{I}$ denote the set of all problem instances for the \IMAB problem with $k$ arms. 
The policy regret of an algorithm $\alg$ for time horizon $T$, is defined as
\begin{equation}
    \label{eq: def policy regret}
    \mathtt{Regret}_{\alg}(T) = \sup_{I\in\mathcal{I}} \big[\opt(I, T) - \mathbb{E}[\alg(I, T)]\big]
\end{equation}
where the expectation is over any randomness in $\alg$.
\end{definition}

%\paragraph{Competitive Ratio:}
In Section \ref{sec: lower bound}, we show that any algorithm for the \IMAB problem must suffer regret that is linear in $T$. 
%This implies that any algorithm, no matter how utilitarian, cannot hope to achieve sub-linear regret. 
This motivates our choice to study the competitive ratio of an algorithm with respect to the offline optimal algorithm. 
We note that competitive ratio is a well-studied notion used to evaluate performance of online algorithms \cite{BORODIN2005,BUCHBINDER2012} and is also studied in the MAB literature \cite{IMMORLICA2019,KESSELHEIM2020,ANDREW2013,BASU2021,DANIELY2019}.
%\vine{Say CR is well-known in online algorithms}
\begin{definition}
\label{def: competitive ratio}
Let $\mathcal{I}$ denote the set of all problem instances for the \IMAB problem with $k$ arms. 
and $\alg$ be an algorithm for the \IMAB problem. % and $\opt(I, T)$ be the offline optimal policy for instance $I$ and time horizon $T$. 
Then the (strict) competitive ratio of $\alg$ for time horizon $T$ is defined as
\begin{equation}
    \compRatio_{\alg}(T) = \inf_{\alpha \in \mathbb{R}} \{ \forall I\in\mathcal{I},~~ \alpha \cdot \alg(I,T) \geq  \opt(I,T) \}
\end{equation}
%The competitive ratio of $\alg$ with respect to $\opt$, denoted $\compRatio_{\alg}(T)$ is defined as the ratio
%\begin{equation}
%    \compRatio_{\alg}(T) = \sup_{I\in\mathcal{I}}\frac{\opt(T)}{\alg(T)}
%\end{equation}
\end{definition}
\vis{We will henceforth refer to this as the competitive ratio of an algorithm. To lower bound the competitive ratio of an algorithm $\alg$ by $\alpha$, it is sufficient to provide an instance $I$ such that $\frac{\opt(I,T)}{\alg(I,T)} \geq \alpha$.
Similarly, to upper bound the competitive ratio of an algorithm $\alg$ by $\alpha$, it is sufficient upper bound $\frac{\opt(I,T)}{\alg(I,T)}\leq \alpha$ for all $I\in \mathcal{I}$.}
Naturally, the goal of the decision-maker is to design an algorithm with a small competitive ratio.
Finally, we note that although we have defined $\rew_i(N)$ and $\alg(T)$ as a discrete sum, in some of our proofs, we use definite integrals (area under the curves  defined by $f_i$) to approximate the value of the discrete sum. This approximation does not affect our results. Please refer to Appendix \ref{app: approximation of sum by integral} for a detailed justification.

%\input{sec_OptimalStrategy}
%\footnote{Note that the competitive ratio can equivalently be defined as $\inf_{I\in\mathcal{I}}\frac{\alg(T)}{\opt(T)}$, with the goal of the decision-maker now being to maximize it.}

\section{Lower Bound and Sub-Optimality of Round Robin}

\label{sec: lower bound}
In this section, we begin by proving the hardness of the \IMAB problem.
In particular, we show that for any time horizon $T$, there is an instance such that any algorithm for the \IMAB problem suffers a regret that is linear in $T$ and, in fact, has competitive ratio $\Omega(k)$. 
%, that does not know the time horizon $T$ beforehand, 
This implies that, even an algorithm that solely wants to maximize its cumulative reward, without any fairness consideration towards the arms, must suffer linear regret. 
%A natural question at this point is: Can we design an algorithm that maximizes reward while guaranteeing some equitable allocation of resources (arm pulls) to the arms? 
%A straightforward candidate that is inherently equitable is the round robin (RR) algorithm.
We also show that the competitive ratio of the simple round-robin algorithm ($\rr$) is $\Theta(k^2)$, and hence $\rr$, even though it equally distributes pulls towards the arms is sub-optimal for the decision-maker.

\begin{wrapfigure}[9]{r}{0.4\textwidth}
  \begin{center}
    \includegraphics[width=\linewidth]{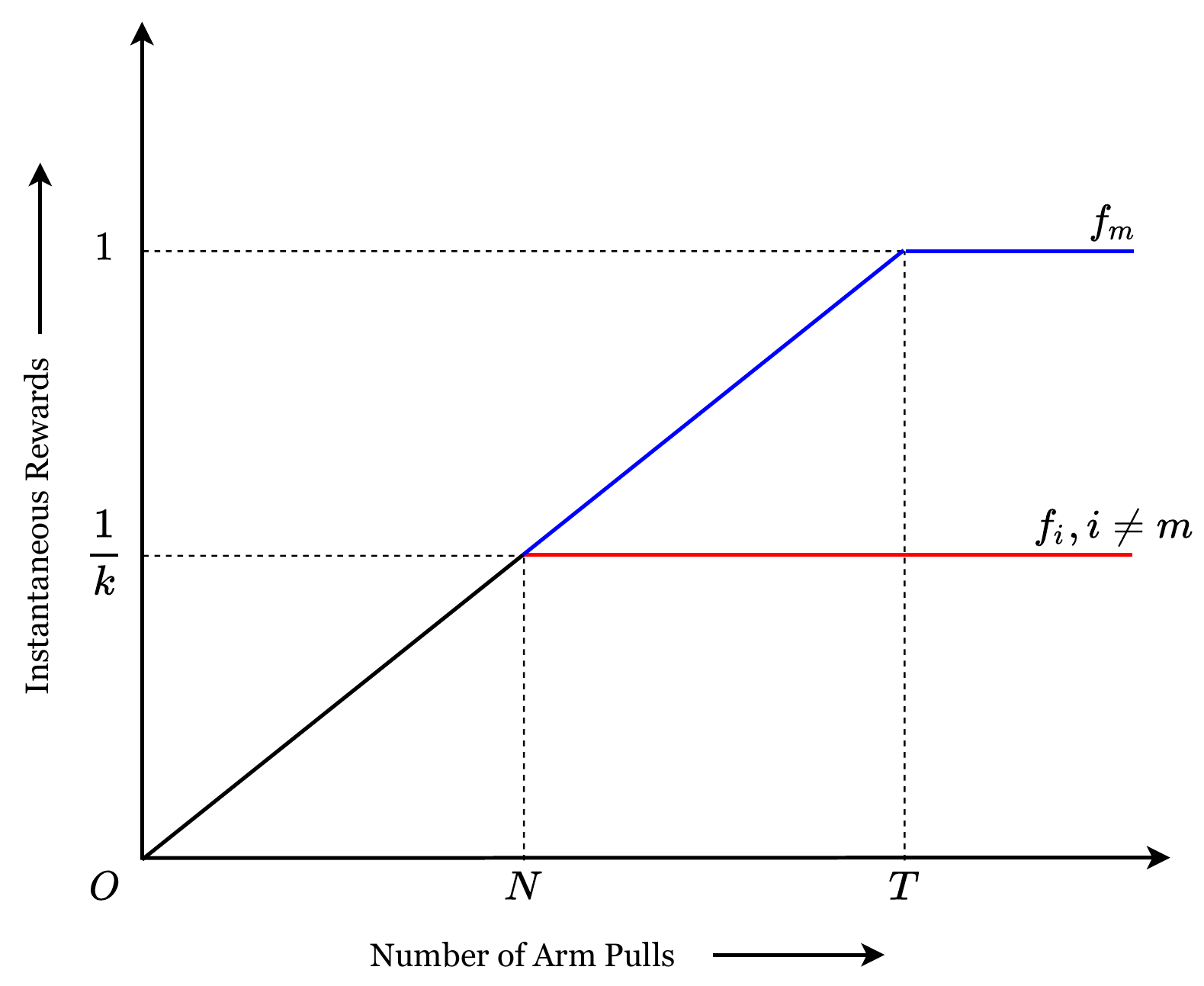}
  \end{center}
  \caption{Instance $I_m$ for Lower Bound}
  \label{fig: lower bound ex main}
\end{wrapfigure}

\paragraph{Lower Bound:} In Theorem \ref{thm: regret lower bound}, we show that for any algorithm $\alg$ and any time horizon $T$, there exists a problem instance $I = \instance$, such that the competitive ratio $\compRatio_\alg(T) \geq \frac{k}{2}$. 
The lower bound shows that we cannot hope to give anytime guarantees with competitive ratio $o(k)$ for the \IMAB problem.
%\vine{Write a couple of line of conclusion here from the fairness perspective. Say that mow much ever an algorithm is utilitarian it suffers $\Omega(T)$ regret.}

\begin{restatable}{theorem}{RegretLowerBound}
\label{thm: regret lower bound}
Let $\alg$ be an algorithm for the \IMAB problem with $k$ arms. Then, for any time horizon $T$, there exists a problem instance defined by the reward functions $(f_1,f_2,\ldots,f_k)$ such that
\begin{enumerate}[label=(\alph*)]
    \item $\mathtt{Regret}_\alg(T) \geq c\cdot T$, for some constant $c$,
    \item $\compRatio_\alg(T) \geq \dfrac{k}{2}$. %$\dfrac{\opt(T)}{\alg(T)} \geq \dfrac{k}{2}$ when $T$ is a multiple of $k$ and  otherwise.
\end{enumerate}
\end{restatable}
The proof of our lower bound also shows that even when $T$ is known to the algorithm, an instance-independent sublinear regret is not possible. 
Note that \citet{HEIDARI2016} give an instance-dependent asymptotically sublinear regret.  
In contrast, we seek an instance-independent anytime regret guarantee (i.e., one that does not depend on the parameters of problem instance). 
We prove this lower bound by constructing a family of $k$ \IMAB instances and showing that no algorithm can achieve sub-linear regret and $o(k)$ competitive ratio on all $k$ instances. 
We briefly provide intuition about the construction of these problem instances and defer the detailed proof to Appendix \ref{app subsec: lower bound}.
Let $N = \lceil T/k \rceil$. The $k$ problem instances are as follows: For $m \in [k]$, instance $I_m$ is such that for all arms $i\neq m$,

\begin{minipage}{.45\textwidth}
  \begin{align*}
    f_i(n) = \begin{cases}
    \vspace{2mm}\dfrac{n}{kN} & \text{If} ~~n \leq N\\ 
    \dfrac{1}{k} & \text{If} ~~n > N \\ 
    \end{cases}
\end{align*}
\end{minipage}%
\begin{minipage}{0.1\textwidth}
  and
\end{minipage}%
\begin{minipage}{.4\textwidth}
  \begin{align*}
    f_m(n) = \begin{cases}
    \vspace{2mm}\dfrac{n}{kN} & \text{If} ~~n \leq kN\\ 
    1 & \text{If} ~~n > kN \\ 
    \end{cases}
\end{align*}
\end{minipage}

See Figure \ref{fig: lower bound ex main} for a depiction of instance $I_m$. An algorithm cannot differentiate between the arms in instance $I_m$ until the optimal arm $m$ has been pulled at least $N$ times. Given the construction of the reward functions, at least one arm, say arm $j$, would be pulled $< N$ times by any algorithm. Consequently, the algorithm will suffer linear regret on instance $I_j$.
%\vine{Mention that our lower bound shows that even when $T$ is known to the algorithm, an instance independent sublinear regret is not possible. Note that
%\tocite\ give an instance dependent sublinear regret.}
%We showed that even an algorithm that solely wants to maximize its cumulative reward, without any fairness consideration towards the arms, must suffer linear regret. 
%A natural question that arises next is: Can we design an algorithm that maximizes reward while guaranteeing some equitable allocation of resources (arm pulls) to the arms? 
%A straightforward candidate that is inherently equitable is the round robin algorithm.
%We next analyze how well it performs in terms of its competitive ratio with respect to the offline optimal.

%\section{A Naive Approach: Round Robin Algorithm}
%\label{sec: round robin}
\paragraph{Round Robin:}The $\rr$ algorithm sequentially pulls arms $1$ to $k$, and at the end of $T$ rounds, for any $T\in \mathbb{N}$ ensures that each arm is pulled at least $\lfloor T/k\rfloor$ times irrespective of the reward obtained from the arm.
Although, this ensures equal distribution of arm pulls,
%In the context of our motivating example, this translates to allocating the amount to a different neighbourhood every year in a sequential manner and repeating this process.
%This equal distribution is the most equitable.
%However, this could be sub-optimal since it does not incentivize the communities to improve their utilization of the allocated resources.
in Theorem \ref{thm: round robin performance}, we show that $\rr$ is sub-optimal in terms of its competitive ratio.

\begin{restatable}{theorem}{RoundRobinAnalysis}
\label{thm: round robin performance}
Let $\rr$ denote the round robin algorithm. %and $\opt$ denote the optimal algorithm for time horizon $T$. 
Then, $8k^2 \geq \compRatio_\rr(T) \geq \frac{k^2}{2}.$
\end{restatable}

The proof is in Appendix \ref{app subsec: round robin}. The first inequality in the above theorem says that the competitive ratio of $\rr$ is at most $8k^2$, whereas the second inequality shows that our analysis for $\rr$ is tight (up to constants). 
%We provide a brief sketch of the proofs below and refer the reader to Appendix \ref{app: round robin analysis} for the details.
%In this section, we showed that even an algorithm that solely wants to maximize its cumulative reward, without any fairness consideration towards the arms, has competitive ratio $\Omega(k)$. We also showed that the competitive ratio of $\rr$ is sub-optimal.
%We posed the requirement of an optimal algorithm that ensures some equitable distribution of pulls among the arms.
%We also showed that RR, which is equitable towards the arms is sub-optimal for the decision maker.
%
In the next section, we propose an algorithm whose competitive ratio is optimal up to constants, and which allocates the pulls to arms in a manner that ensures each arm attains its true potential given sufficient time.

\section{Optimal Algorithm for Improving Bandits}
\label{sec: optimal algorithm horizon unaware}
%We state our algorithm in this section, which is our main contribution, and show that 
\vis{In this section, we propose an algorithm that \emph{mitigates disparity} due to lack of opportunities while achieving the \emph{best possible cumulative reward} at any time.}
%The main objective of our work is to design an algorithm that \emph{mitigates disparity} while \emph{maximizing the cumulative reward}.
%We next show that these two objectives, which intuitively may appear to be conflicting in the \IMAB model, are in fact in perfect harmony with each other.
%\vin{In this section, we propose an anytime algorithm (Algorithm \ref{algo: horizon unaware improving bandits}) for \IMAB problem and show that its cumulative reward in the online setting is optimal up to constants. 
%In particular, we show that the cumulative reward accrued by our algorithm in the online setting is optimal up to constants. 
%We also show that our proposed algorithm, given enough time, ensures that each arm is pulled sufficient number of times so that it reaches its true potential.
We first give an intuitive idea about how our algorithm works.
Then, we formally state the two-way guarantee that our algorithm provides; first, the tight guarantee for the cumulative reward, and second, the mitigation of disparity by helping arms reach their true potential.
In Section \ref{subsec: algo and guarantee}, we state our algorithm along with the main results (Theorems \ref{theorem: k competitive algo proof} and \ref{theorem: fairness of optimal algo}).
In Section \ref{subsec: proof sketch of cr}, we provide a proof sketch of Theorem \ref{theorem: k competitive algo proof} along with the supporting Lemmas.
Finally, in Section \ref{subsec: proof sketch fairness}, we give a proof sketch of Theorem \ref{theorem: fairness of optimal algo}.
%\vine{Do we say about the proof sketches here?}
%We begin by explaining the algorithm and formally stating the above guarantees as part of Theorems \ref{theorem: k competitive algo proof} and \ref{theorem: fairness of optimal algo} in Section \ref{subsec: algo and guarantee}, and then give the proof sketches of the two theorems in Sections \ref{subsec: proof sketch of cr} and \ref{subsec: proof sketch fairness}.

\subsection{Algorithm and its Guarantees}
\label{subsec: algo and guarantee}
Our proposed algorithm is in Algorithm \ref{algo: horizon unaware improving bandits}.
\vis{Throughout this section, we use $\alg$ to denote Algorithm \ref{algo: horizon unaware improving bandits} unless stated otherwise.
We remark that, in addition to having strong performance guarantees, $\alg$ is a simple (in terms of the operations used) and efficient (in terms of time complexity) algorithm.}
Before stating the theoretical guarantees of our algorithm we provide an intuitive explanation of how it works. 
For each arm $i\in [k]$, recall $N_i(t)$ denotes the number of times arm $i$ has been pulled until (not including) time step $t$ and for notational convenience we use $N_i$ to denote the number of times $\alg$ pulls arm $i$ in $T$ time steps. 
Initialization is done by pulling each arm twice (Step 3). 
This lets us compute the rate of change of the reward function between the first and second pulls of each arm, i.e., $\Delta_i(2) = f_i(2) - f_i(1)$ (as defined at Step 7).
This takes $2k$ time steps.
\vis{At each time step $t> 2k$, we let $i^*_t$ denote the arm that has been pulled maximum number of times so far, i.e., $i^*_t \in \arg\max_{i\in[k]} N_i(t)$.}
%At each time step $t > 2k$, we determine the maximum number of arm pulls among all arms, $N_{i^*} = \arg\max_{i\in [k]} N_i(t)$. 
\vis{Then, for every arm $i\in[k]$, we compute an optimistic estimate of its cumulative reward had it been pulled $N_{i^*_t}(t)$ times, denoted by $p_i(t)$.}
%\vine{Here $N_i^*$ is confusing because of $N_i(T)$ replaced with $N_i$. This should be $N_i^*(t)$}
\vis{The optimistic estimate, $p_i(t)$, is computed by adding the actual cumulative reward obtained from arm $i$ in $N_i(t)$ pulls, 
denoted $\rew_i(N_i(t))$, and the maximum cumulative reward that can be obtained from the arm in additional $N_{i^*_t}(t) - N_i(t)$ pulls if it continues to increase at the current rate, $\Delta_i(N_i(t))$.%, for the additional $N_{i^*} - N_i(t)$ time steps. 
}
We then pull an arm with the largest value of $p_i(t)$. 
Ties are first broken based on the minimum value of $N_i(t)$ and further ties can be broken arbitrarily.

%\begin{wrapfigure}{l}{0.6\textwidth}
%\begin{wrapfigure}{L}{0.5\textwidth}
%\begin{minipage}{0.5\textwidth}
\begin{algorithm}[ht!]
  \caption{Horizon-Unaware Improving Bandits}
  \label{algo: horizon unaware improving bandits}
 \SetAlgoLined
% \KwIn{$[k]$}
 \textbf{Initialize:} \\ 
   $N_i(0) = 0$ for all arms $i\in [k]$ \hspace{4cm}\textcolor{gray}{Number of pulls of arm $i\in[k]$}\\
   Pull each arm twice\\
   $t = 2k+1$ \hspace{6cm}\textcolor{gray}{Current time step after $2k$ arm pulls}\\
   $N_i(t) = 2$ for all arms $i\in [k]$\\
   \For{$t = 2k+1, \ldots$,T} { 
     $\Delta_i(N_i(t)) = f_i(N_i(t)) - f_i(N_i(t) - 1)$\\
     $i^*_t \in \arg\max_{i\in[k]} N_i(t)$\\
     \For{$i=1,2,\ldots,k$}{
        $p_i(t) = \rew_i(N_i(t)) + \sum_{n=1}^{{N_{i^*}(t)} - N_i(t)} \left[f_i(N_i(t)) + n\cdot\Delta_i\left(N_i(t)\right)\right]$\\
     }
     $C = \arg\max_{i\in [k]} p_i(t)$\\
     Pull arm $i_t = \arg\min_{i\in C} N_i(t)$\\
     \For{$i=1,2,\ldots,k$}{
        $N_i(t+1) = N_i(t) + \mathbbm{1}\{i_t = i\}$\\
        
    }
    }
\end{algorithm}
%\end{minipage}
%\end{wrapfigure}
%\end{wrapfigure}

%Let $\alg$ denote Algorithm \ref{algo: horizon unaware improving bandits}. 
Our goal is to provide an upper bound on $\compRatio_{\alg}(T)$. 
%The following theorem, which is one of the key technical contributions of our work, proves that Algorithm \ref{algo: horizon unaware improving bandits} has an optimal (up to constants) competitive ratio.
\vin{The following theorem, which is one of the key technical contributions of our work, proves that Algorithm \ref{algo: horizon unaware improving bandits} has $O(k)$ competitive ratio, and from Theorem \ref{thm: regret lower bound} in Section \ref{sec: lower bound} it follows that the competitive ratio of our algorithm is optimal (up to constants).}
%------------------------------------------ Begin Theorem 1-----------------------------------------------------

\begin{restatable}{theorem}{OptAlgoCompetitiveRatio}
\label{theorem: k competitive algo proof}
%For $\alg$ being Algorithm \ref{algo: horizon unaware improving bandits},
Competitive Ratio of $\alg$ is $O(k)$. In particular,
$
\compRatio_{\alg}(T) \leq 32k.
$
Further, the time complexity of $\alg$ is $O(k\log k)$ per time step.
\end{restatable}
\vis{The per round time complexity of $\alg$ follows from the $\arg\max$ operation (steps 8) performed at each time step, which is standard in MAB literature.
This shows that our algorithm, in addition to being simple, is also efficient.
}
The proof of the above theorem relies on some neat attributes of our algorithm and the class of reward functions.
We discuss some of these in Section \ref{subsec: proof sketch of cr}. 
We next show in Theorem \ref{theorem: fairness of optimal algo} that $\alg$ ensures that each arm reaches its true potential given sufficient time.
\begin{restatable}{theorem}{OptAlgoFairness}
\label{theorem: fairness of optimal algo}
%Let $\alg$ denote Algorithm \ref{algo: horizon unaware improving bandits}. 
For an arm $i\in [k]$, let $a_i = \lim_{N\rightarrow \infty} f_i(N)$. 
%Then for $\alg$ being Algorithm \ref{algo: horizon unaware improving bandits}, and for every $\epsilon \in (0,a_i)$, 
\vis{Then, for every $\varepsilon\in (0,a_i]$, }there exists $T\in \mathbb{N}$ such that $\alg$ ensures that $a_i - f_i(N_{i}(T)) \leq \varepsilon \,.$
\end{restatable}
%\vine{$\epsilon \in (0,a_i)$}
Theorem \ref{theorem: fairness of optimal algo} shows that all arms reach arbitrarily close to their true potential given sufficient time. 
In particular, this shows that our algorithm \vin{mitigates the initial disparities in the arms due to lack of opportunities} by enabling the arms to reach their true potential given sufficient time.
%In particular, this shows that our algorithms \vin{mitigates the initial disparities in the arms due to lack of opportunities} enables the arms to reach their true potential given sufficient time.
We give a proof sketch of the above theorem in Section \ref{subsec: proof sketch fairness}.
%\vine{Change the writing below the fairness thm. We have to bring mitigating disparity here.}
%\todo{Write a few lines to elaborate upon the theorem.}
%------------------------------------------ End Theorem 2-----------------------------------------------------

\subsection{Proof Sketch of Theorem \ref{theorem: k competitive algo proof}}\label{subsec: proof sketch of cr}
%Without loss of generality, let $N_1 \geq N_2 \geq \ldots \geq N_k$. 
The proof of the theorem hinges upon Lemmas  \ref{lemms: first arm to cross N pulls}, \ref{lem: rew N over rew T for arm i}, and \ref{lemma: rew N by opt T for arm i} and Corollary \ref{corr: opt N by opt T for arm i} stated below. 
We elaborate upon these lemmas along with the proof sketches for a few of them and then explain how the proof is completed using these lemmas. 
The complete proof of Theorem \ref{theorem: k competitive algo proof} along with the proofs of the lemmas and corollaries is in Appendix \ref{app: missing proofs optimal algo}.

%------------------------------------------------ Begin Lemma 1 ----------------------------------------------------
Let $I = \instance$ be an arbitrary instance of the \IMAB problem and let $T$ be the time horizon.
To upper bound the competitive ratio of our algorithm at $T$, it is sufficient to upper bound $\frac{\opt(I,T)}{\alg(I,T)}$ (since instance $I$ has been chosen arbitrarily). 
Throughout, and without loss of generality, assume $N_1 \geq N_2 \geq \ldots \geq N_k$. 
We begin with a crucial lemma which captures an important feature of our algorithm: the first arm to cross $N$ pulls has to be the optimal arm for the time horizon $N$, that is, as per Proposition \ref{prop: single arm optimal policy}, it has to be the arm that maximizes the cumulative reward for horizon $N$.
%In Lemma \ref{lemms: first arm to cross N pulls}, we show that our algorithm satisfies a crucial property
This property is of key importance in proving the optimality of our algorithm with respect to the competitive ratio. 
\begin{restatable}{lemma}{FirstArmToCrossNPulls}
\label{lemms: first arm to cross N pulls}
If arm $i \in [k]$ is the first arm to cross $N$ pulls, i.e., to be pulled $N+1$-th time, then $$\rew_i(N) = \opt(I,N).$$
\end{restatable}
%\begin{proof}[Proof Sketch]
%The detailed proof can be found in Appendix \ref{app: missing proofs optimal algo}.
%This lemma states that 
%That is, as per Proposition \ref{prop: single arm optimal policy}, it has to be the arm that maximizes the cumulative reward for horizon $N$.
We note that if our algorithm runs for $T$ time steps then the above lemma holds for any $N$ between $1$ and $T$.
The proof of the above lemma relies on how we compute $p_i(t)$, the optimistic estimate of the cumulative reward of arm $i$, at each time step.
We remark here that previous works that study the \IMAB problem assume that the horizon $T$ is known to the algorithm beforehand.
This significantly simplifies the problem of estimating the optimistic estimate of the cumulative reward of any arm using a linear extrapolation.
This also allows certain arms to be eliminated based on these estimates.
However, such an approach is not possible in the anytime setting.
In fact, in Theorem \ref{theorem: fairness of optimal algo}, we show that our algorithm keeps pulling an arm till it is improving.

Next, we state Lemma \ref{lem: rew N over rew T for arm i} that lower bounds the ratio ${\rew_i(N)}/{\rew_i(T)}$, which is the ratio of the cumulative reward of pulling arm $i$ for $N$ pulls to that of pulling it for $T$ pulls, for each arm $i\in[k]$. 
Part $(a)$ of the following lemma considers the case when $N > T/2$ and part $(b)$ looks at the case when $N \leq T/2$.
\begin{restatable}{lemma}{RewNOverRewTForArmi}
\label{lem: rew N over rew T for arm i}
For each arm $i \in [k]$,
\begin{enumerate}[label=(\alph*)]
    \item $\dfrac{\rew_i(\alpha T)}{\rew_i(T)} \geq \dfrac{1}{5}$ ~~~~~~~~~for $\alpha \geq \dfrac{1}{2}$ ,
    \item $\dfrac{\rew_i(\alpha T/k)}{\rew_i(T)} \geq \dfrac{16\alpha^2}{25k^2}$ ~~for $0 \leq \alpha \leq \dfrac{k}{2}$.
\end{enumerate}
\end{restatable}
%
%\begin{proof}[Proof Sketch]
The proof of the above lemma relies on the properties of $f_i$, in particular, the properties that 
$f_i$'s are monotonically increasing, bounded in $[0,1]$, and have decreasing marginal returns.
To prove part ($a$), we first show that $\rew_i(\alpha T)$ can be lower bounded by the area of triangle defined by $O$, $E$, and $B$ in Figure \ref{fig: Lemma Case A} which is equal to $\frac{\alpha T f_i(\alpha T)}{2}$. 
Further, for $\alpha\geq1/2$ we show that $\rew_i(T) \leq \frac{5T}{4}f_i(\alpha T)$ (see Claim \ref{claim: alpha greater than half, reward in T upper bound} in Appendix \ref{app: missing proofs optimal algo}). 
This gives us part ($a$) of the lemma.
To prove part ($b$), we show that $\rew_i(\alpha T/k)$ is lower bounded by the area of the triangle defined by $O$, $E$, 
and $B$ in Figure \ref{fig: Lemma Case B} which is equal to $\frac{\alpha^2T^2\slopeOC}{2k^2}$, 
where $\slopeOC$ is the slope of the line segment passing through points $O$ and $E$ in Figure \ref{fig: Lemma Case B}. 
Using arguments leveraging certain geometric properties satisfied by $f_i$, we show that $\rew_i(T) \leq \frac{25T^2\slopeOC}{32}$ (see Claim \ref{claim: alpha less than k by 2, rew T upper bound} in Appendix \ref{app: missing proofs optimal algo}). 
This gives us part ($b$) of the lemma.
%\end{proof}
\begin{figure}
\centering
    \begin{minipage}{0.5\textwidth}
    \centering
    \includegraphics[scale=0.5]{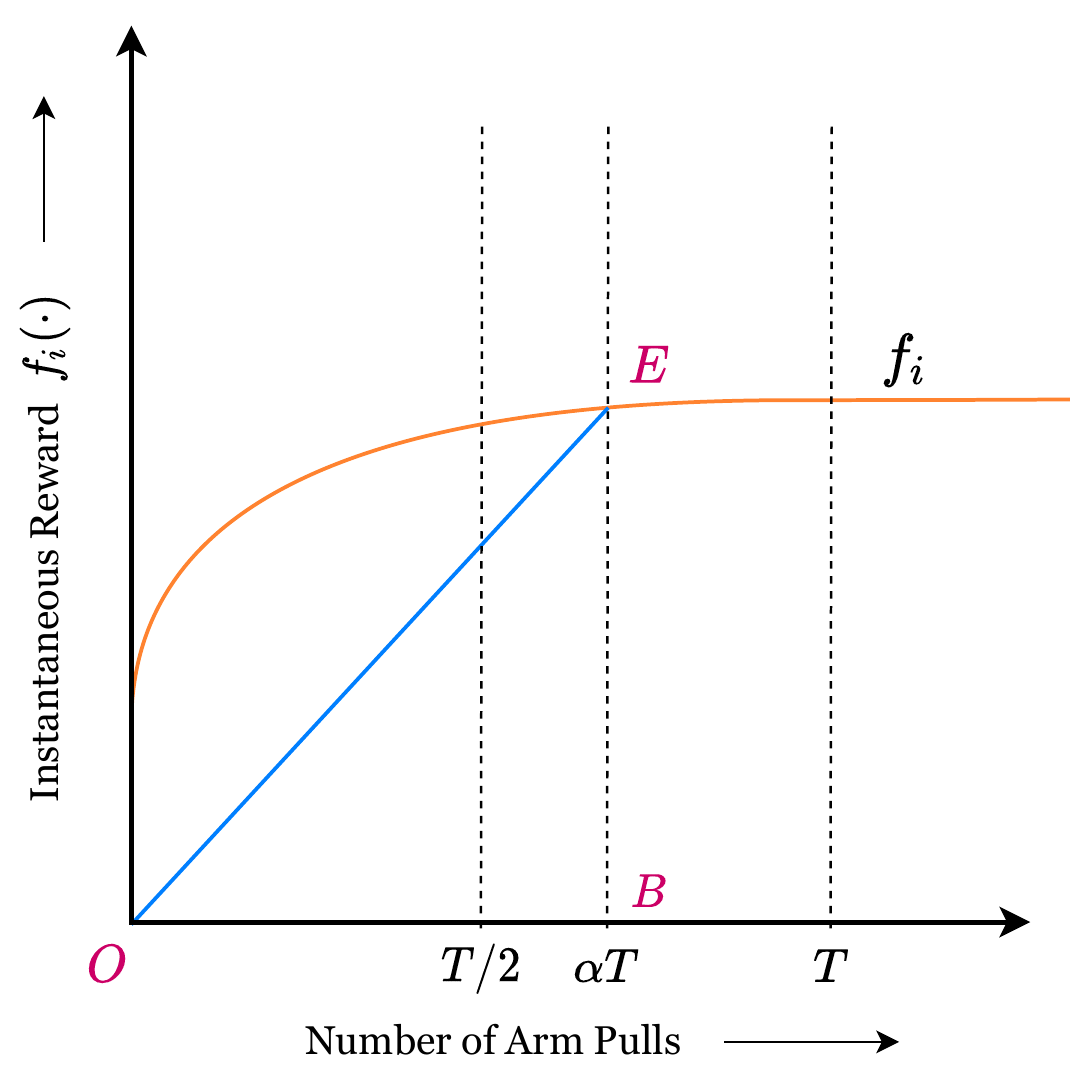}
    \captionof{figure}{$\alpha > 1/2$}
    \label{fig: Lemma Case A}
    \end{minipage}%
    \begin{minipage}{0.5\textwidth}
    \centering
    \includegraphics[scale=0.5]{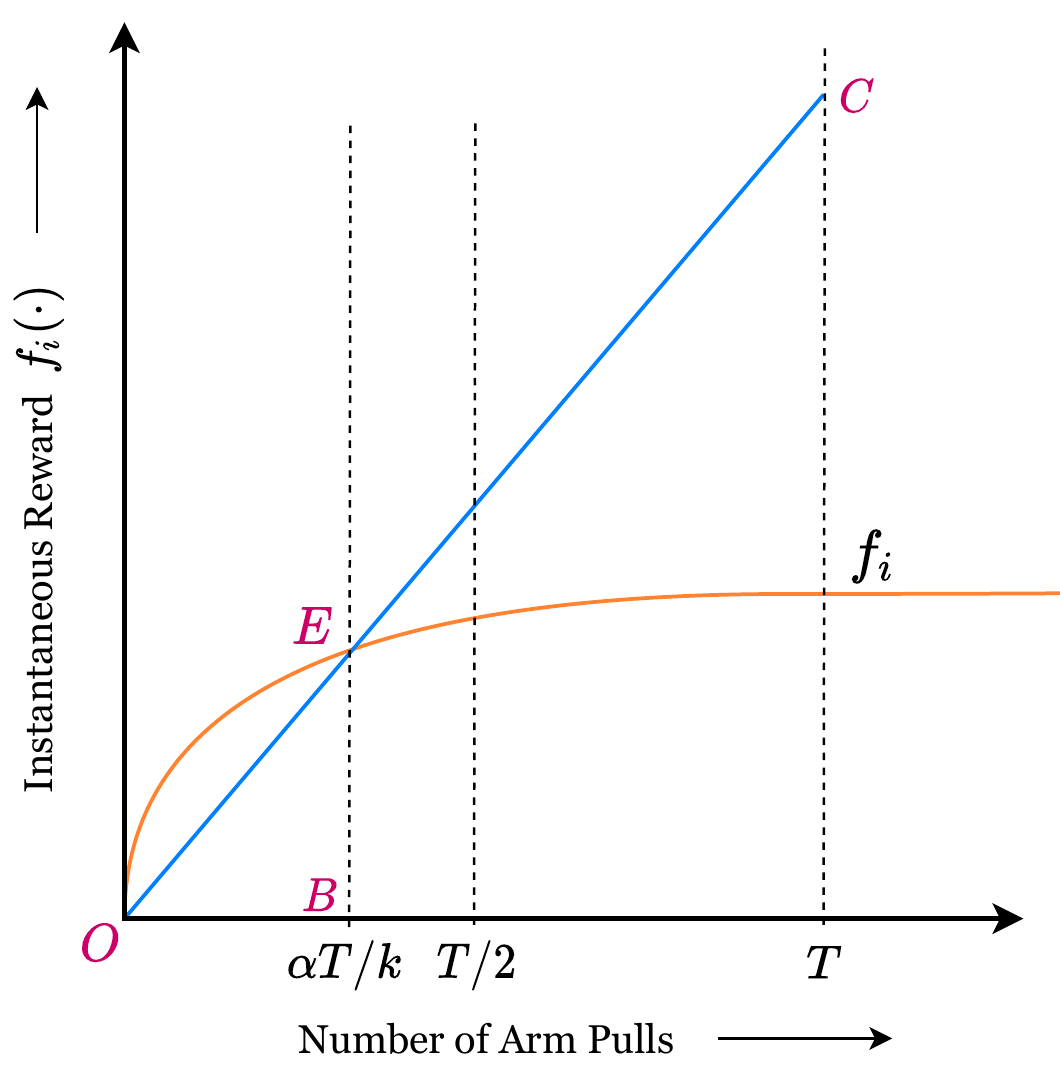}
    \captionof{figure}{$0 < \alpha \leq k/2$}
    \label{fig: Lemma Case B}
    \end{minipage}
\end{figure}
%
%------------------------------------------ End Lemma 1 -----------------------------------------------------
%
%------------------------------------------ Begin Corollary -----------------------------------------------------
Next, we have the following interesting corollary to Lemma \ref{lem: rew N over rew T for arm i} which compares the optimal rewards at $T$ and $N$ where $N$ spans values in $\{1, \ldots, T\} $ depending on the value of $\alpha$. 
\begin{restatable}{corollary}{OptNoverOptTforArmi}
\label{corr: opt N by opt T for arm i}
For any finite time horizon $T$, we have
\begin{enumerate}[label=(\alph*)]
    \item $\dfrac{\opt(I,\alpha T)}{\opt(I,T)} \geq \dfrac{1}{5}$ ~for~ $\dfrac{1}{2} \leq \alpha \leq 1$ ,
    \item $\dfrac{\opt(I,\alpha T/k)}{\opt(I,T)} \geq \dfrac{16\alpha^2}{25k^2}$ ~for~ $0 < \alpha \leq \dfrac{k}{2}$. 
\end{enumerate}
\end{restatable}
%The detailed proof of this corollary can be found in Appendix \ref{app: missing proofs optimal algo}. 
%\begin{proof}[Proof Sketch] %We provide a proof sketch here.
The proof of the above corollary uses Proposition \ref{prop: single arm optimal policy} and Lemma \ref{lem: rew N over rew T for arm i}. From Proposition \ref{prop: single arm optimal policy}, we know that the optimal policy for time horizon $T$ pulls a single arm.
Let $j^*_T \in [k]$ denote this arm.
Further, note that $\opt(I,\alpha T) \geq \rew_{j^*_T}(\alpha T)$, by definition of $\opt(I,\alpha T)$.
Since, $\opt(I,T) = \rew_{j^*_T}(T)$, part ($a$) follows from part ($a$) of Lemma \ref{lem: rew N over rew T for arm i}.
Part ($b$) is also proved using a similar argument.% but uses part (b) of Lemma \ref{lem: rew N over rew T for arm i}.
%\end{proof}
%------------------------------------------ End Corollary -----------------------------------------------------
%
%------------------------------------------ Begin Lemma 2-----------------------------------------------------

%\end{proof}
%------------------------------------------ End Lemma 2-----------------------------------------------------

%------------------------------------------ Begin Lemma 3-----------------------------------------------------
Now, observe that $\alg(I,T)$, i.e., the cumulative reward of our algorithm after $T$ time steps can be written as the sum of the rewards obtained from each arm.
In particular, $\alg(I,T) = \sum_{i\in[k]} \rew_i(N_i)$, where $N_i$ is the number of times arm $i$ has been pulled in $T$ time steps.
Recall that, our goal is to provide an upper bound on $ \frac{\opt(I,T)}{\alg(I,T)}$, or equivalently, $\frac{\opt(I,T)}{\sum_{i\in[k]} \rew_i(N_i)}$.
The following lemma provides an upper bound on $\frac{\opt(I,T)}{\rew_i(N_i)}$ in terms of only $N_i$ and the time horizon $T$, when $N_1 \leq T/2$. %, i.e., the maximum number of arm pulls among all arms is less than $T/2$.
We handle the (easier) case of $N_1 > T/2$ separately (see case 1 of proof of Theorem \ref{theorem: k competitive algo proof} in Appendix \ref{app: missing proofs optimal algo}).

\begin{restatable}{lemma}{RewNOverOPTTForArmi}
\label{lemma: rew N by opt T for arm i}
If $N_1 \leq T/2$ then for any arm $i\in [k]$,
$$
\dfrac{\opt(I,T)}{\rew_i(N_i)} \leq \dfrac{200T^2}{N_i^2}.
$$
\end{restatable}
\vis{The proof of the above lemma is non-trivial and requires intricate case analysis using different properties of our algorithm and the reward functions.}
The proof of the above lemma uses Lemmas \ref{lem: rew N over rew T for arm i} and \ref{lemma: rew N by opt T for arm i}, and Corollary \ref{corr: opt N by opt T for arm i}. 
Finally, with all the components in place, we provide a brief proof sketch of Theorem \ref{theorem: k competitive algo proof}.
%------------------------------------------ End Lemma 3-----------------------------------------------------

%------------------------------------------ Begin Theorem Proof-----------------------------------------------------
\begin{proof}[Proof Sketch of Theorem \ref{theorem: k competitive algo proof}]
We look at the following two cases: 1) $N_1 > T/2$, and 2) $N_1 \leq T/2$.

Case 1 implies that arm $1$ is the first, and hence the only arm to cross $T/2$ pulls. 
From Lemma \ref{lemms: first arm to cross N pulls}, we get, $\rew_1(T/2) = \opt(I,T/2)$.
Therefore, $\alg(I,T) \geq \rew_1(T/2) = \opt(I,T/2)$.
 Hence, we obtain, $\opt(I,T)/\alg(I,T)
\leq \opt(I,T)/\opt(I,T/2)
\leq 5 \leq 200k$.
Here, the second inequality follows from Corollary \ref{corr: opt N by opt T for arm i}. 
Since the above inequality holds for any instance $I$, we get $\compRatio_\alg(T) \leq 200k$.

For Case 2, we use Lemma \ref{lemma: rew N by opt T for arm i} and obtain
$$ \dfrac{\sum_{i\in[k]} \rew_i(N_i)}{\opt(I,T)} \geq \dfrac{\sum_{i\in[k]}N_i^2}{200T^2} \geq \frac{(\sum_{i \in [k]}  N_i/\sqrt{k})^2}{200 T^2}
=\frac{T^2}{200 k T^2}=\frac{1}{200k}.$$
%\compRatio_\alg(T) \leq
Here, the second inequality follows from Cauchy-Schwarz inequality (see Observation \ref{obs: cauchy schwarz} in Appendix \ref{app: missing proofs optimal algo}).
%, we show that $\dfrac{\alg(I,T)}{\opt(I,T)} \geq \dfrac{1}{32k}$.
This implies $\dfrac{\opt(I,T)}{\alg(I,T)} \leq 200k$. Since this holds for an arbitrary instance $I\in \mathcal{I}$, we have $\compRatio_\alg(T) \leq 200k$, i.e., it is $O(k)$.
\end{proof}
%------------------------------------------ End Theorem Proof -----------------------------------------------------

\subsection{Proof Sketch for Theorem \ref{theorem: fairness of optimal algo}}
\label{subsec: proof sketch fairness}
The theorem is proved using Lemma \ref{lemma: fairness of optimal algo} stated below. 
We give a proof sketch for the lemma and then explain how the proof of the theorem is completed using the lemma.
The detailed proof of Lemma \ref{lemma: fairness of optimal algo} and Theorem \ref{theorem: fairness of optimal algo} can be found in Appendix \ref{app: proof of fairness}.
In Lemma \ref{lemma: fairness of optimal algo}, we show that $\alg$ pulls an arm finitely many times only if the arm stops improving. 
\begin{restatable}{lemma}{FairnessLemma}
\label{lemma: fairness of optimal algo}
Let $L_i = \max_{t\in \mathbb{N}}\{N_{i}(t)\}$ for all $i\in [k]$. Then for any $i\in [k]$, $L_i$ is finite implies that $\Delta_i(L_i) = 0$.
\end{restatable}
$L_i$ as defined in the above lemma captures the number of times the algorithm pulls arm $i$ as $T$ tends to infinity.
It is easy to see that there is at least one arm $i\in [k]$ such that $L_i$ is not finite, and hence, the property holds for this arm vacuously. 
The proof of the lemma argues via contradiction that the property has to be satisfied for all the arms.
Suppose there is an arm $j\in [k]$ such that $L_j$ is finite but $\Delta_j(L_j) \neq 0$.  
Then we consider a time horizon larger than when arm $j$ was pulled for the $L_j$-th time and use the definition of the optimistic estimate to show that such an arm is indeed pulled again,  contradicting the assumption. 

The theorem is proved using the above lemma as follows. Suppose $L_i$ as defined in Lemma \ref{lemma: fairness of optimal algo} is finite for an arm $i\in [k]$. 
Then the marginal decreasing property of the reward functions ensures that arm $i$ has reached its true potential,
i.e, $f_i(L_i) = a_i$. Further, if $L_i$ is not finite then the arm is pulled infinitely many times, and hence again
from the properties of the reward functions we have that for every $\varepsilon \in (0,a_i]$, there exists $T\in \mathbb{N}$ such that $\alg$ ensures the following:
$a_i - f_i(N_{i}(T)) \leq \varepsilon \,.$
\section{Conclusion and Future Work}\label{sec: conclusion}
We studied the \IMAB problem in horizon-unaware setting, and proposed an algorithm that achieves optimal competitive ratio at any time.
An interesting feature of our algorithm from the fairness perspective is that it keeps pulling an arm till it reaches its true potential.
This enables the arms to reach their true potential and mitigates the disparities that may exist between the potentials of the arms due to lack of opportunities. We further showed that the objective of maximizing cumulative reward is aligned

The \IMAB model assumes the reward functions are monotonically increasing, bounded, and have decreasing marginal returns property.
Our results leverage these properties to show the optimality of our algorithm.
\citet{LINDNER2021} study the single-peaked bandit model where the reward functions have a single peak under the assumption that the time horizon is known to the algorithm and provide asymptotic regret guarantees. It would be interesting to see if our ideas and techniques can be extended to the single-peaked bandits setting to obtain anytime guarantees on the regret and the competitive ratio.

%\section*{Ethics Statement}
%We give an algorithm with strong theoretical guarantees for the \IMAB model.
%Although this model provides an abstraction for many real-world applications, in practice, some properties the \IMAB model may not always be satisfied.
%Extending our work to be robust to such discrepancies is an interesting future direction.
\section*{Acknowledgements}
Vishakha Patil is grateful for the support of a Google PhD Fellowship.
Arindam's research is  supported by Pratiksha Trust Young Investigator Award,
Google India Research Award, and Google ExploreCS Award.
\bibliographystyle{plainnat}
\bibliography{references}

\newpage
\appendix
\section{Additional Preliminaries}
\label{app: additional prelims}

\subsection{Approximation of Discrete Sum by Definite Integral}
\label{app: approximation of sum by integral}
\begin{lemma}
\label{prop: approximation of sum by int}
Let $f_i$ be the reward function corresponding to some arm $i\in[k]$, satisfying the properties described in Section \ref{sec: model} (monotonically increasing, bounded in $[0,1]$, and decreasing marginal returns).
Then, $$\int_{t=0}^{T} f_i(t)\,dt \leq \rew_i(T) \leq \int_{t=0}^{T+1} f_i(t)\,dt.$$ 
%$$\rew_i(T) - 1 \leq \int_{0}^T f_i(t)\cdot dt \leq \rew_i(T).$$
\end{lemma}
\begin{proof}
We prove this using simple results from calculus (see Chapter 5 of \cite{HUGHES2020}).
We know that $f_i: [0,T] \rightarrow [0,1]$ is a monotonically increasing function.
Let $P = \{[x_0, x_1],[x_1,x_2],\ldots,[x_{n-1},x_n]\}$ be an arbitrary partition of the interval $[0,T]$, where $$0 = x_0 < x_1 < x_2 < \ldots < x_n = T.$$

Then, the \emph{Left Riemann Sum} $L$ of $f_i$ over $[0,T]$ with partition $P$ is defined as
\begin{equation}
    \label{eq: general left riemann sum}
    L = \sum_{j=1}^n f_i(x_{j-1}) (x_j - x_{j-1})
\end{equation}
and the \emph{Right Riemann Sum} $R$ of $f$ over $[0,T]$ with partition $P$ is defined as
\begin{equation}
    \label{eq: general right riemann sum}
    R = \sum_{j=1}^n f_i(x_{j}) (x_j - x_{j-1})
\end{equation}
Now, let $P = \{[0,1], [1,2], [2,3],\ldots, [T-1, T]\}$. 
Substituting this in Eqs. \ref{eq: general left riemann sum} and \ref{eq: general right riemann sum}, we obtain,
\begin{equation}
    \label{eq: left riemann}
    L = \sum_{j=1}^T f_i(j-1)
\end{equation}
and,
\begin{equation}
    \label{eq: right riemann}
    R = \sum_{j=1}^T f_i(j)
\end{equation}
%Suppose $f$ corresponds to the reward function of some arm $i$. 
Given that $f_i$ corresponds to the reward function of arm $i$, $R$ in Eq. \ref{eq: right riemann} is the cumulative reward of pulling arm $i$ for $T$ times, i.e., $R = \rew_i(T)$. 

%Further, $R - L = f_i(T) - f_i(0) \leq 1$. 
%The inequality holds because $f_i$ is bounded between $0$ and $1$.

Since $f_i$ is monotonically increasing, we know that $$L \leq \int_{0}^{T} f_i(t) \,dt \leq R.$$
See Chapter 5 of \cite{HUGHES2020} for the above. Hence, from this we can conclude that 
$$\rew_i(T) - f_i(T) + f_i(0) = L \leq \int_{0}^{T} f_i(t) \,dt \leq R = \rew_i(T).$$

Further, note that
\begin{align*}
    \rew_i(T) &\leq \int_{t=0}^T f_i(t)\,dt + [f_i(T) - f_i(t)] \tag{From the first inequality above}\\
    & \leq \int_{t=0}^T f_i(t)\,dt + f_i(T) \tag{Since $f_i$'s are non-negative}\\
    & \leq \int_{t=0}^T f_i(t)\,dt + \int_{t=T}^{T+1} f_i(t)\,dt \tag{Using the Left Reimann Sum between $T$ and $T+1$}\\
    & = \int_{t=0}^{T+1} f_i(t)\,dt
\end{align*}
Hence, we can conclude that
$$\int_{t=0}^{T} f_i(t)\,dt \leq \rew_i(T) \leq \int_{t=0}^{T+1} f_i(t)\,dt.$$
%That is, the error in the continuous approximation of the cumulative reward is at most $1$.

\end{proof}

\subsection{Policy Regret vs. External Regret \cite{HEIDARI2016}}
\label{app: example from heidari}
Consider an instance $I = \instance$ of the \IMAB problem with $k=2$, and $f_1(n) = n/10$ and $f_2(n) = 0.1$ for all $n > 1$.
Now, let $\alg2$ be an arm that always pulls arm $2$.
The external regret of this algorithm is zero since, at every time step, it pulls the arm with the highest instantaneous reward.
However, the optimal offline algorithm for this problem always pulls arm $1$ at every time step.
Hence, the policy regret of $\alg2$ increases linearly with $T$.
\section{Proof of Proposition \ref{prop: single arm optimal policy}}
\label{app: proof of offline optimal strategy}
This proposition has been proved in \cite{HEIDARI2016}.
We provide an alternate proof here.
We first prove the following claim for the case when the number of arms $k =2$.
Using this, we then prove the proposition using an exchange argument and mathematical induction.

\begin{claim}
\label{claim: two arm optimal offline}
Consider an \IMAB instance $I = \instance$ with $k = 2$ and let $T$ be the time horizon.
Let $\alg$ be an algorithm that plays a single arm $j^*_T$ at each time step, where 
\begin{equation}
    j^*_T = \arg\max_{i\in[k]} \sum_{t=1}^T f_i(t)
\end{equation}
Let $\algprime$ be any other algorithm for the $\IMAB$ problem. Then,
$$ \alg(T) \geq  \algprime(T).$$
\end{claim}
\begin{proof}
Recall that the cumulative reward of any algorithm for the \IMAB problem depends only on the number arm pulls of each arm and not the order in which the arms are pulled.
We prove the claim by induction on the number of time steps at which $\alg$ and $\algprime$ differ.
Let $n$ denote this quantity.
Assume, without loss of generality, that $j^*_T = 1$.

\textbf{Base case: $n=1$}\\
Then, we know that $\algprime$ pulls arm $1$ for $T-1$ time steps, and arm $2$ only one time.
Then,
\begin{align*}
    \alg(T) - \algprime(T) &= \sum_{n=1}^T f_1(n) - \sum_{n=1}^{T-1} f_1(n) - f_2(1)\\
    & = f_1(T) - f_2(1)
\end{align*}
Now suppose $f_1(T) - f_2(T) < 0$. 
This implies that 
\begin{equation}
    \label{eq: prop one at T less than 2 at 1}
    f_1(T) < f_2(1)
\end{equation}
Hence, we get

\begin{equation}
    f_1(r)  \leq f_1(T)  
     < f_2(1)  
     \leq f_2(s)  \  \ \text{ for all } \ \ r, s \in [T].
     \label{propOne:EqOne}
\end{equation}
%\begin{align*}
%    f_1(r) &\leq f_1(T) \tag{For all $r \leq T$, since $f_i$'s are monotonically increasing}\\
%    & < f_2(1) \tag{From Eq. \ref{eq: prop one at T less than 2 at 1}}\\
%    & \leq f_2(s) \tag{For all $s \geq 1$, since $f_i$'s are monotonically increasing}
%\end{align*}
Now, let $\alg_2$ denote another algorithm that plays arm $2$ for all time steps $t\in[T]$. 
From above, we can conclude that
\begin{align*}
    \alg(T) = \sum_{s=1}^T f_1(s)  < \sum_{s=1}^T f_2(s) = \alg_2(T)
\end{align*}
But this contradicts the definition of $j^*_T (= 1)$. 

\textbf{Inductive Hypothesis: } Let $\algprime$ be an algorithm that differs from $\alg$ at exactly $n$ time steps, such that $\algprime(T) \leq \alg(T)$.
Note that $\alg(T) = \sum_{s=1}^T f_1(s)$ and $\algprime(T) = \sum_{s=1}^{T-n} f_1(s) + \sum_{s=1}^n f_2(s).$
From the inductive hypothesis, we get
\begin{align}
\label{eq: prop inductive hypothesis A}
    \alg(T) - \algprime(T) &= \sum_{s=1}^T f_1(s) - \sum_{s=1}^{T-n} f_1(s) - \sum_{s=1}^n f_2(s)\nonumber\\
    & = \sum_{s=T-n+1}^T f_1(s) - \sum_{s=1}^n f_2(s)\nonumber\\
    & \geq 0 
\end{align}
The last inequality follows from the induction hypothesis.
Now, let $\algtwo$ be another algorithm that differs from $\alg$ at exactly $n+1$ time steps.
Then, we can conclude that $\algtwo$ plays arm $1$ for $T-(n+1)$ pulls and arm $2$ for $n+1$ pulls.
From this, we obtain
\begin{align*}
    \alg(T) - \algtwo(T) & = \sum_{s=1}^T f_1(s) - \sum_{s=1}^{T-(n+1)} f_1(s) - \sum_{s=1}^{n+1} f_2(s) \tag{From the definition of $\alg$ and $\algtwo$}\\
    & = \sum_{s=T-n}^T f_1(s) - \sum_{s=1}^{n+1} f_2(s) \tag{Re-arranging the terms above}\\
    & = \overbrace{\underbrace{\sum_{s=T-n+1}^T f_1(s) - \sum_{s=1}^n f_2(s)}_\text{B} + f_1(T-n) - f_2(n+1)}^\text{A}
\end{align*}
From the inductive hypothesis (Eq. \ref{eq: prop inductive hypothesis A}) we know that $B \geq 0$.
Suppose $A < 0$.
These two statements together imply that $f_1(T-n) - f_2(n+1) < 0$, or equivalently, $f_1(T-n) < f_2(n+1)$.
Hence, we obtain
\begin{equation}
    \underbrace{(T-(n+1)) f_1(T-n)}_\text{C} < \underbrace{(T-(n+1)) f_2(n+1)}_\text{D} \tag{Since $T \geq 0$ and $n < T$}
\end{equation}
Further, $A < 0$ implies that
\begin{align*}
    \sum_{s=T-n+1} f_1(s) + f_1(T-n) & < \sum_{s=1}^n f_2(s) + f_2(n+1)\\
    \implies \sum_{s=T-n+1} f_1(s) + (T-n)f_1(T-n) & < \sum_{s=1}^n f_2(s) + (T-n) f_2(n+1) \tag{Adding $C$ on LHS and $D$ on RHS}
\end{align*}
Finally, we obtain
\begin{align*}
    \alg(T) &=\sum_{s=1}^T f_1(s) \\
    & \overset{(i)}{\leq} (T-n)f_1(T-n) + \sum_{s=T-n+1}^{T} f_1(s)\\
    & < \sum_{s=1}^n f_2(s) + (T-n) f_2(n+1)\\
    & \overset{(ii)}{\leq} \sum_{s=1}^T f_2(s)\\
    & = \alg_2(T)
\end{align*}
Inequalities $(i)$ and $(ii)$ follow from the monotonicity of $f_1$ and $f_2$. 
Recall that $\alg_2$ is an algorithm that always pulls arm $2$.
Hence, we obtain $\alg(T) < \alg_2(T)$ which is a contradiction to the definition of $j^*_T (= 1)$.

Hence, we conclude that $A = \alg(T) - \algtwo(T) \geq 0$.
This concludes the proof of Claim \ref{claim: two arm optimal offline}.
\end{proof}

Using the above claim, we now prove the proposition.

\begin{proof}[Proof of Proposition \ref{prop: single arm optimal policy}]
Let $j^*_T \in \arg\max_{i\in [k]} \sum_{s=1}^T f_i(s)$.
Assume, without loss of generality, that $j^*_T = 1$.
Let $\alg$ be an algorithm that pulls arm $j^*_T = 1$ for all time steps $t\in [T]$ and let $\algprime$ be any other algorithm for this problem.
Then, to prove the statement of the proposition, it is sufficient to prove that $\algprime(T) \leq \alg(T)$.

Suppose this is not true, i.e., suppose there is an algorithm $\algprime$ such that $\algprime(T) > \alg(T)$.
For each arm $i\in[k]$, let $M_i$ denote the number of times $\algprime$ pulls arm $i$ in $T$ time steps.
Then, we have $\sum_{i\in[k]} \sum_{s=1}^{M_i} f_i(s) > \sum_{s=1}^T f_1(s)$.
Now, consider arms $k-1$ and $k$.
Let $T_k = M_{k-1}+M_k$.
Further, let $i^*_k \in \arg\max_{i \in \{k-1,k\}} \sum_{s=1}^{T_k} f_i(s)$.
Assume, without loss of generality, that $i^*_{k-1} = k-1$.
Then, from Claim \ref{claim: two arm optimal offline}, we get
\begin{equation}
    \label{eq: exchange argument 1}
    \sum_{s=1}^{M_{k-1}} f_{k-1}(s) + \sum_{s=1}^{M_{k}} f_{k}(s) \leq \sum_{s=1}^{T_k} f_{k-1}(s)
\end{equation}
Now, we can construct another algorithm $\algprime_k$ from $\algprime$ by replacing every pull of arm $k$ with a pull of arm $k-1$.
Note that every arm $i\in [k]\setminus\{k-1,k\}$ is pulled the same number of times by $\algprime$ and $\algprime_k$.
And, $\algprime_k$ pulls arm $k-1$ for $M_{k-1} + M_k$ pulls and does not pull arm $k$ at all.
From Eq. \ref{eq: exchange argument 1}, we obtain $\algprime(T) \leq \algprime_k(T)$.

%We next look at $\algprime_k$ and arms $k-2$ and $k-1$.
We can repeatedly apply the same process as described in the previous paragraph: 
In particular, for any $\ell \in \{1,2,\ldots,k\}$, we can define an algorithm $\algprime_\ell$ that pulls arms $1$ to $\ell-2$ for the same number of rounds as $\alg$, pulls arm $\ell - 1$ for $\sum_{s=\ell-1}^k M_s$ number of time steps, and does not pull arms $\ell$ to $k$.
Repeating this argument, which hinges on Claim \ref{claim: two arm optimal offline}, we can construct a sequence of algorithms $\algprime_{k-1}$, $\algprime_{k-2},\ldots,\algprime_3$ such that
\begin{equation}
    \label{eq: exchange argumetn algo sequence}
    \alg(T) < \algprime(T) \leq \algprime_k \leq \algprime_{k-1} \leq \ldots \leq \algprime_3.
\end{equation}

Note that $\algprime_3$ plays arm $1$ for $M_1$ pulls and arm $2$ for $T-M_1$ pulls.
Now, we need to apply the above argument one final time by considering arms $1$ and $2$ and $T_2 = T$.
Now, $i^*_2 \in \arg\max_{i\in\{1,2\}} \sum_{s=1}^{T} f_i(s)$.
We now have two cases: 1) the set $\arg\max_{i\in\{1,2\}} \sum_{s=1}^{T} f_i(s)$ contains a single arm, in which case it has to be arm $1$ (from the definition of $J^*_T$), or 2) the set contains both arms $1$ and $2$.

Consider case $1$.
As before, replacing all pulls of arm $2$ by arm $1$, we obtain
$$\alg(T) < \algprime(T) \leq \algprime_k(T) \leq \ldots \leq \algprime_3(T) \leq \alg(T),$$
which contradicts our assumption that there is an algorithm $\algprime$ such that $\algprime(T) > \alg(T)$.

Next, consider case $2$.
If $i^*_2 = 1$, then the argument in case $1$ provides the same contradiction.
Suppose $i^*_2 = 2$.
Then, as before, replacing all pulls of arm $1$ by arm $2$, we obtain
$$\alg(T) < \algprime(T) \leq \algprime_k(T) \leq \ldots \leq \algprime_3(T) \leq \algprime_2(T),$$
where $\algprime_2 = \alg_2$ is an algorithm that pulls arm $2$ for all time steps $T$.
This contradicts our definition of arm $j^*_T$ which in turn contradicts the assumption that there is an algorithm $\algprime$ such that $\algprime(T) > \alg(T)$.

This concludes the proof of the proposition.
\end{proof}

\section{Missing Proofs from Section \ref{sec: lower bound}}
\label{app: lower bound}
In this section, we prove Theorem \ref{thm: regret lower bound} in Section \ref{app subsec: lower bound} and Theorem \ref{thm: round robin performance} in Section \ref{app subsec: round robin}.
We denote the area of a figure with $n$ endpoints $\{A_1, \ldots, A_{n-1}, A_n\}$ as $\Ar(A_1 \ldots A_{n} A_1)$. 
Given a line segment $AB$, without loss of generality we use $AB$ to denote the length of the line segment $AB$, unless specified otherwise.
We use $\angle AEB$ to denote the angle between the edges $AE$ and $EB$ in a triangle $AEB$.

\subsection{Proof of Theorem \ref{thm: regret lower bound}: Lower Bound}
\label{app subsec: lower bound}

\RegretLowerBound*
\begin{proof}
%We prove this theorem in two cases: one, where $k$ divides $T$, and two, where $k$ does not divide $T$. 
%We do this for ease of exposition and the techniques are the same.\\
%\textbf{Case 1: $k \mid T$}
Let $T$ be the time horizon. 
Fix a \vis{deterministic} algorithm $\alg$.
\vis{Note that we prove the lower bound for any deterministic algorithm.
Since a randomized algorithm can be thought of as a distribution over the set of all deterministic algorithms, our lower bound also extends to the randomized setting in expectation.}
Our task is to construct a problem instance on which $\alg$ suffers $\Omega(T)$ regret. 
Let $N = \lceil T/k \rceil$. 
Consider the set of $k$ instances $\{\mathcal{I}_1,\ldots,\mathcal{I}_k\}$ such that, for instance $\mathcal{I}_m$, for all arms $i \neq m$:
\begin{align*}
    f_i(n) = \begin{cases}
    \vspace{2mm}\dfrac{n}{T} & \text{If} ~~n \leq N\\ 
    \dfrac{1}{k} & \text{If} ~~n > N \\ 
    \end{cases}
\end{align*}
and for arm $m$,
\begin{align*}
    f_m(n) = \begin{cases}
    \vspace{2mm}\dfrac{n}{T} & \text{If} ~~n \leq T\\ 
    1 & \text{If} ~~n > T \\ 
    \end{cases}
\end{align*}

We note here that the instances above are slightly different from those in the main body ($f_i(n)$ is $n/T$ instead of $n/kN$). 
We emphasize that the guarantee in the main body still holds. 
However, the slight change in the instance definition here gives us a tighter guarantee.
Observe that after $T$ time steps have elapsed, there exists an arm $\ell\in[k]$ such that $N_\ell \leq T/k.$ 
We then show that $\alg$ suffers $\Omega(T)$ regret on instance $\mathcal{I}_\ell$.
Without loss of generality, let arm $k$ be such an arm, i.e., arm $k$ is such that $N_k < N.$
Now, we consider instance $\mathcal{I}_k$.
Observe that the maximum reward that can be obtained from each of the $T$ pulls of $\alg$ is not more than $\frac{1}{k}$. 
This is because, any arm $i\neq k$ cannot give instantaneous reward $> 1/k$ and arm $k$ is not pulled more than $N$ times. 
Then we get,
\begin{align}
    \alg(T) & \leq T\cdot \frac{1}{k} \nonumber\\
    & = T/k \label{eq: lower bound alg}
\end{align}
Further, from Proposition \ref{prop: single arm optimal policy}, we know that the optimal policy for horizon $T$, that knows the reward functions beforehand would always pull arm $k$ and obtain reward
\begin{align}
    \opt(T) & = \sum_{n=1}^{T} \frac{n}{T} \nonumber\\
    & = \frac{T+1}{2} \label{eq: lower bound opt}
\end{align}
From this we obtain,
\begin{align*}
    \opt(T) - \alg(T) &\geq \frac{T+1}{2} - \frac{T}{k} \tag{From Equations \eqref{eq: lower bound alg} and \eqref{eq: lower bound opt}}\\
    & = \Big(\frac{1}{2} - \frac{1}{k}\Big) T + \frac{1}{2}\\
    & \geq \frac{k-2}{2k}\cdot T\\
    & \geq \frac{T}{6} \tag{For $k > 2$}
\end{align*}
This proves the first part of our theorem. Next, we have
\begin{align*}
    \dfrac{\opt(T)}{\alg(T)} & \geq \dfrac{k}{2}\cdot\dfrac{T+1}{T}\\
    & \geq \dfrac{k}{2}
\end{align*}
This concludes our proof.
\end{proof}

\begin{figure}[t]
\centering
\begin{minipage}{.5\textwidth}
  \centering
  \includegraphics[width=\linewidth]{Figures/LowerBoundEx.pdf}
  \captionof{figure}{Lower Bound Instance}
  \label{fig: lower bound ex}
\end{minipage}%
\begin{minipage}{.5\textwidth}
  \centering
  \includegraphics[width=\linewidth]{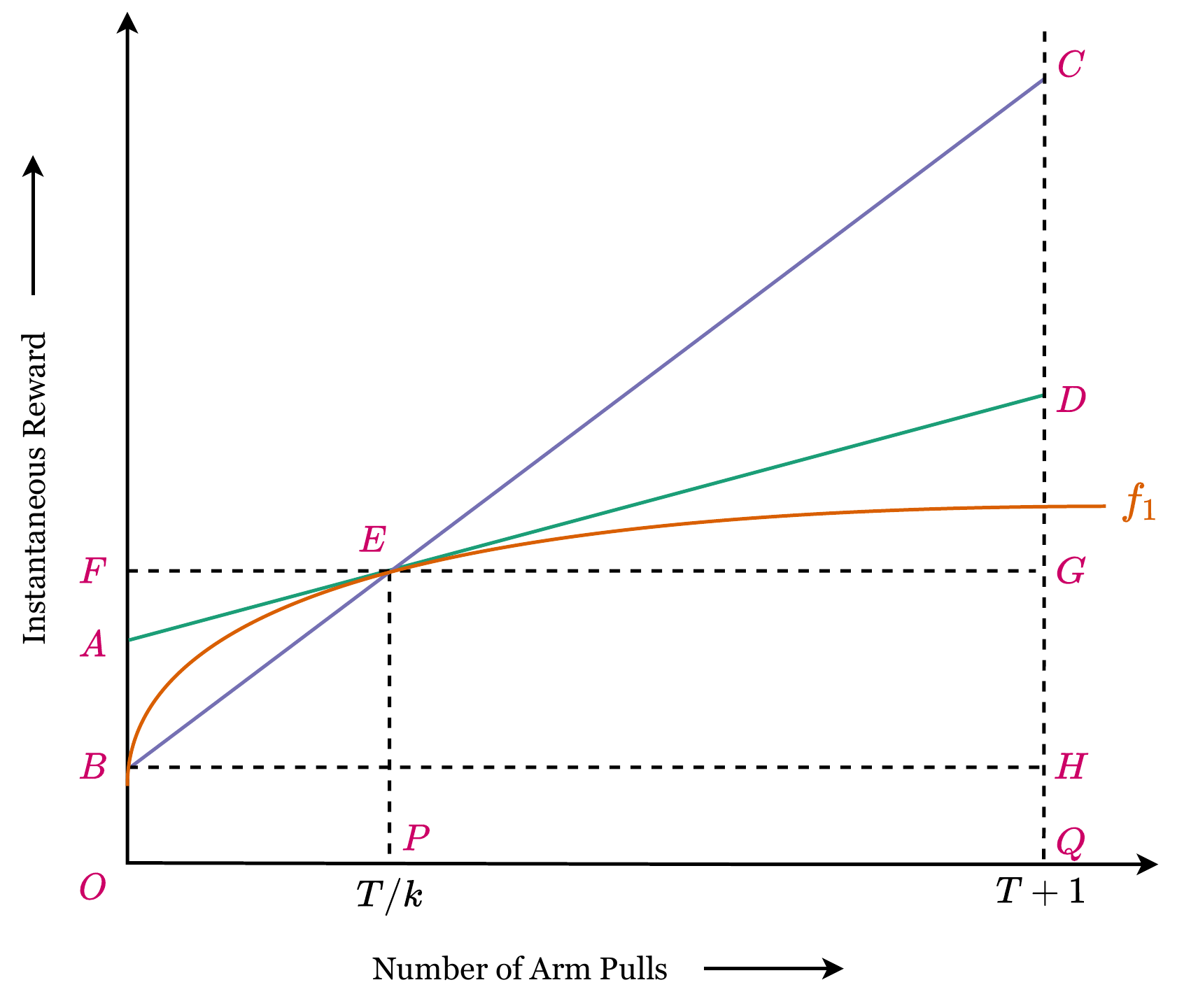}
  \captionof{figure}{Reward Function of Arm $1$ ($f_1$)}
  \label{fig: round robin}
\end{minipage}
\end{figure}

\subsection{Proof of Theorem \ref{thm: round robin performance}: Round Robin ($\rr$) Analysis}
\label{app subsec: round robin}
%\section{Analysis of Round Robin Algorithm}
%\label{app: round robin analysis}
\RoundRobinAnalysis*
\begin{proof}
We analyse the competitive ratio of $\rr$ for $T\geq 2k$.
First, we show that $\compRatio_{\rr}(T)\geq \frac{k^2}{2}$. 
We show this by describing an instance $I$ of the \IMAB problem for which $\frac{\opt(I,T)}{\rr(I,T)} \geq \frac{k^2}{2}$. %and show that our analysis of RR is tight up to constants.
The instance $\instance$ is given as follows:
\begin{align*}
    f_1(n) = \begin{cases}
    \vspace{2mm}\dfrac{n}{T} & \text{If} ~~n \leq T\\ 
    1 & \text{If} ~~n > T \\ 
    \end{cases}
\end{align*}
and for arms $2,\ldots,k$,
%\begin{align*}
$    f_i(n) =     0  ~~~\text{for} ~~n \geq 1$.
%\end{align*}
It is easy to see that $\opt(I,T)$ pulls arm $1$ for all $T$ time steps and obtains reward $$\opt(I,T) = \sum_{t=1}^T \frac{n}{T} = \frac{T+1}{2}.$$
Let $\rr(I,T)$ denote the total reward of the round robin algorithm on instance $I$ in $T$ rounds. Then, $$\rr(I,T) = \sum_{t=1}^{T/k} \frac{n}{T}= \frac{T+k}{2k^2}.$$
Then, we get
\begin{align*}
    \frac{\opt(I,T)}{\rr(I,T)} & \geq \frac{T+1}{T+k}\cdot k^2
  \geq \frac{T}{2T} \cdot k^2 \tag{For $T \geq k$}\\
    & = \frac{k^2}{2}
\end{align*}

Next, we show that $\compRatio_{\rr}(T)\leq  8k^2$. Consider an arbitrary time horizon $T$. For simplicity, we assume that $T$ is a multiple of $k$. From Proposition \ref{prop: single arm optimal policy} we know that the optimal policy consists of pulling a single arm for all time steps from $1$ to $T$. Without loss of generality assume that the $\instance$ is such that the optimal policy always pulls arm $1$. Let $\opt(I,T)$ denote the cumulative reward from $T$ pulls of arm $1$. 

\begin{comment}
\begin{figure}
    \centering
    \includegraphics[scale=0.5]{Figures/roundrobin.pdf}
    \caption{$f_1$: Reward Function of Arm $1$}
    \label{fig: round robin}
\end{figure}
\end{comment}

In Figure \ref{fig: round robin}, \vis{the curve labeled $f_1$ denotes the reward function of arm 1.}
%the orange curve denotes the reward function of arm $1$ and is labelled $f_1$. 
We know that $\rr$ will pull arm $1$ exactly $T/k$ times in $T$ time steps. Then,
\begin{align}
    \rr(I,T) & = \sum_{i=1}^{k} \sum_{n=1}^{T/k} f_i(n) \nonumber\\
    & \geq \sum_{n=1}^{T/k} f_1(n) \tag{Since the rewards are non-negative}\nonumber\\
    & \geq \Ar(OBEPO) \nonumber\\
    & = \frac{1}{2}\cdot \frac{T}{k} \cdot (f_1(T/k) - f_1(0)) + f_1(0)\cdot \frac{T}{k} \nonumber\\
    & = \Big(\frac{T}{2k}\Big){(f_1(0) + f_1(T/k))} \label{eq: alg lower bound}
\end{align}
%\text{Area under the line segment BE}
where, $\Ar(OBEPO)$ is the area of the polygon $OBEP$ (area under the line segment $BE$), $f_1(0)$ corresponds to \vis{the y-coordinate of point} $B$ and $f_1(T/k)$ to \vis{the y-coordinate of point} $E$ in Figure \ref{fig: round robin}. 
Since there is no ambiguity, we henceforth use $f$ to denote $f_1$. 
Next, we will upper bound the value of $\opt(I,T)$. 
Since arm $1$ is the optimal arm at time $T$, from Figure \ref{fig: round robin}, we can conclude that $$\opt(I,T) \leq \Ar(OADQO)\, ,$$ 
where $\Ar(OADQO)$ is the area of the polygon $OADQ$.
%We remark here that, we use $\Ar(OADQO)$ computed at $x=T+1$ to provide the upper bound instead of $T$ (see Appendix \ref{app: approximation of sum by integral}).
To bound $\opt(I,T)$, we will use the following claim.
\begin{claim}
\label{claim: similar triangle claim}
In Figure \ref{fig: round robin}, $\Ar(OADQO) \leq \Ar(OBCQO)$.
%\gan{\texttt{Area}(ADTO) $\leq$ \texttt{Area} (BCTO)}
\end{claim}
\begin{proof}%[Proof of Claim \ref{claim: similar triangle claim}]
To prove the claim, it is sufficient to prove that $\Ar(ABEA) \leq \Ar(DCED)$, where $\Ar(ABEA)$ and $\Ar(DCED)$ denote the area of triangles $ABE$ and $DCE$ respectively. 
%First, note that
Note that $\angle AEB = \angle DEC$ as they are vertically opposite.
Similarly, $\angle BAE = \angle CDE$ as they are alternate interior angles.
%\begin{equation*}
%    \angle \text{AEB} = \angle \text{DEC} \tag{Vertically opposite angles}
%\end{equation*}
%and,
%\begin{equation*}
%    \angle \text{BAE} = \angle \text{CDE} \tag{Alternate interior angles}
%\end{equation*}
%\end{proof}
This implies that triangles $ABE$ and $DCE$ are similar. 
Further, $\angle BEF = \angle CEG$ as they are vertically opposite, and $\angle BFE = \angle CGE$ as they are both $90\degree$.
This implies triangles $BEF$ and $CEG$ are similar, and hence,
$$\frac{FE}{GE} = \frac{BE}{CE}.$$

%\begin{equation}
%\label{eq: sim triangle 1}
%    \triangle \text{ABE} \sim \triangle CDE
%\end{equation}
%Hence, we get that triangles ABE and CDE are similar triangles. Further,
%\begin{equation*}
%    \angle \text{BEF} = \angle \text{CEG} \tag{Vertically opposite angles}
%\end{equation*}
%and,
%\begin{equation*}
%    \angle \text{BFE} = \angle \text{CGE} \tag{Both equal to $90\degree$}
%\end{equation*}
%The above two equations imply that $\triangle$ BEF $\sim \triangle$ CEG. Then, $$\triangle \text{BEF} \sim \triangle \text{CEG} \implies \frac{EF}{EG} = \frac{BE}{EC}.$$
Note that $FE = \frac{T}{k} \leq T - \frac{T}{k} = GE$, which implies that 
\begin{equation}
    \label{eq: BE less than EC}
    BE \leq CE
\end{equation}

We know that triangles $ABE$ and $DCE$ are similar. This implies that $$\frac{AB}{DC} = \frac{BE}{CE} = \frac{AE}{DE}.$$
From Equation \eqref{eq: BE less than EC}, we get $AB \leq DC$, and $AE \leq DE$.
%\begin{equation*}
%    \text{AB} \leq \text{CD}
%\end{equation*}
%and,
%\begin{equation*}
%    \text{AE} \leq \text{GE}
%\end{equation*}
Since all three sides of triangle $ABE$ are at most the corresponding sides of triangle $DCE$, we get that $\Ar(ABEA) \leq \Ar(DCED)$.% Area($\triangle ABE$) $\leq$ Area($\triangle$ CDE).
\end{proof}
%\gan{should we write the proof of claim 2 here? I think we can make it concise so it will not be a distrction to the proof of the Theorem. If not, we shuld mention where the proof of this claim is given. }
The equation of the line passing through points $B$ and $C$ is given by:
\begin{equation*}
    \label{eq: line equation}
    y = \frac{k(f(T/k) - f(0))}{T} x + f(0) %- \frac{k}{T-k} (f(T/k) - f(1))
\end{equation*}
We upper bound $\opt(I,T)$ by providing an upper bound on $\Ar(OBCQO)$.
We remark here that the x-coordinate of $OBCQO$ is actually $T+1$ and not $T$ (see Appendix \ref{app: approximation of sum by integral}).
We do not show both $T$ and $T+1$ in the figure to avoid confusion.
\begin{align}
    \opt(I,T) &\leq \Ar(OBCQO) \tag{From Claim \ref{claim: similar triangle claim}}\nonumber\\
    & = \frac{1}{2}\cdot(T+1)\cdot \text{length}(CH) + f(0)(T+1) \nonumber\\
    & = \frac{k}{2T}\cdot(T+1)^2\cdot [f(T/k) - f(0)] + f(0)(T+1) \nonumber\\
    %& = \frac{k(T-1)^2}{2(T-k)}f(T/k) + \Big((T-1) - \frac{k(T-1)^2}{2(T-k)}\Big)f(1) \nonumber\\
    & \leq \frac{k(T+1)^2}{2T}f(T/k) \label{eq: opt upper bound}
\end{align}
The last inequality holds since $k\geq 2 \implies (T+1) - \frac{k(T+1)^2}{2T} \leq 0$. From Equations \ref{eq: alg lower bound} and \ref{eq: opt upper bound}, we get
\begin{align*}
    \frac{\opt(I,T)}{\rr(I,T)} & \leq \frac{k(T+1)^2 f(T/k)}{2T} \cdot \frac{2k}{T(f(1) + f(T/k))} \\
    & = \frac{k^2(T+1)^2}{T^2} \cdot \frac{f(T/k)}{f(1) + f(T/k)} \\
    & \leq \frac{k^2(T+1)^2}{T^2}\\
    & \leq \frac{4T^2 \cdot k^2}{T^2} \tag{Since $T+1 \leq 2T$}\\
    & \leq 8k^2  \label{eq: CR RR UB}
\end{align*}
Since the above bound holds for an arbitrary instance $I$, $\compRatio_{\rr}(T) \leq 8k^2$.
%\gan{Minor: Should we instead have Claim 2 as $\opt(T) \leq \texttt{Area}$(BCTO)? We can also highlight these regions in figure... }
%----------------------------------------------------------------------------------------------------------------
%\gan{Should we first put the lower bound and then write upper bound? I think lower bound is more important result here? }
\end{proof}

\section{Proof of Theorem \ref{theorem: k competitive algo proof}: Competitive Analysis of Algorithm \ref{algo: horizon unaware improving bandits}}
\label{app: missing proofs optimal algo}
%\gan{Should we move the statement of Theorem 4 here? Also I think the proof of Lemma 6 should appear before Lemma 7 ..}
Throughout without loss of generality, we assume $N_1 \geq N_2 \geq \ldots \geq N_k$, where $N_i$ denotes $N_i(T+1)$ for $i\in [k]$.
Also,  we refer to area under a curve with $n$ endpoints $\{A_1, \ldots, A_{n-1}, A_n\}$ as $\Ar(A_1 \ldots A_{n} A_1)$. 
Given a line segment $AB$, without loss of generality we use $AB$ to denote the length of the line segment $AB$, unless specified otherwise.
We use $\angle AEB$ to denote the angle between the edges $AE$ and $EB$ in a triangle $AEB$.

Theorem \ref{theorem: k competitive algo proof} was stated in Section \ref{subsec: algo and guarantee}, and we restate it here for completeness.
\OptAlgoCompetitiveRatio*
We begin by giving the proofs of Lemma \ref{lemms: first arm to cross N pulls}, Lemma \ref{lem: rew N over rew T for arm i} and its Corollary \ref{corr: opt N by opt T for arm i}, and Lemma \ref{lemma: rew N by opt T for arm i} respectively. Then we complete the proof of the theorem using them.
%%%%%%%%%Figures%%%%%%%%%%%%%%%%%%
\begin{figure}[ht!]
\centering
    \begin{minipage}{0.5\textwidth}
    \centering
    \includegraphics[scale=0.62]{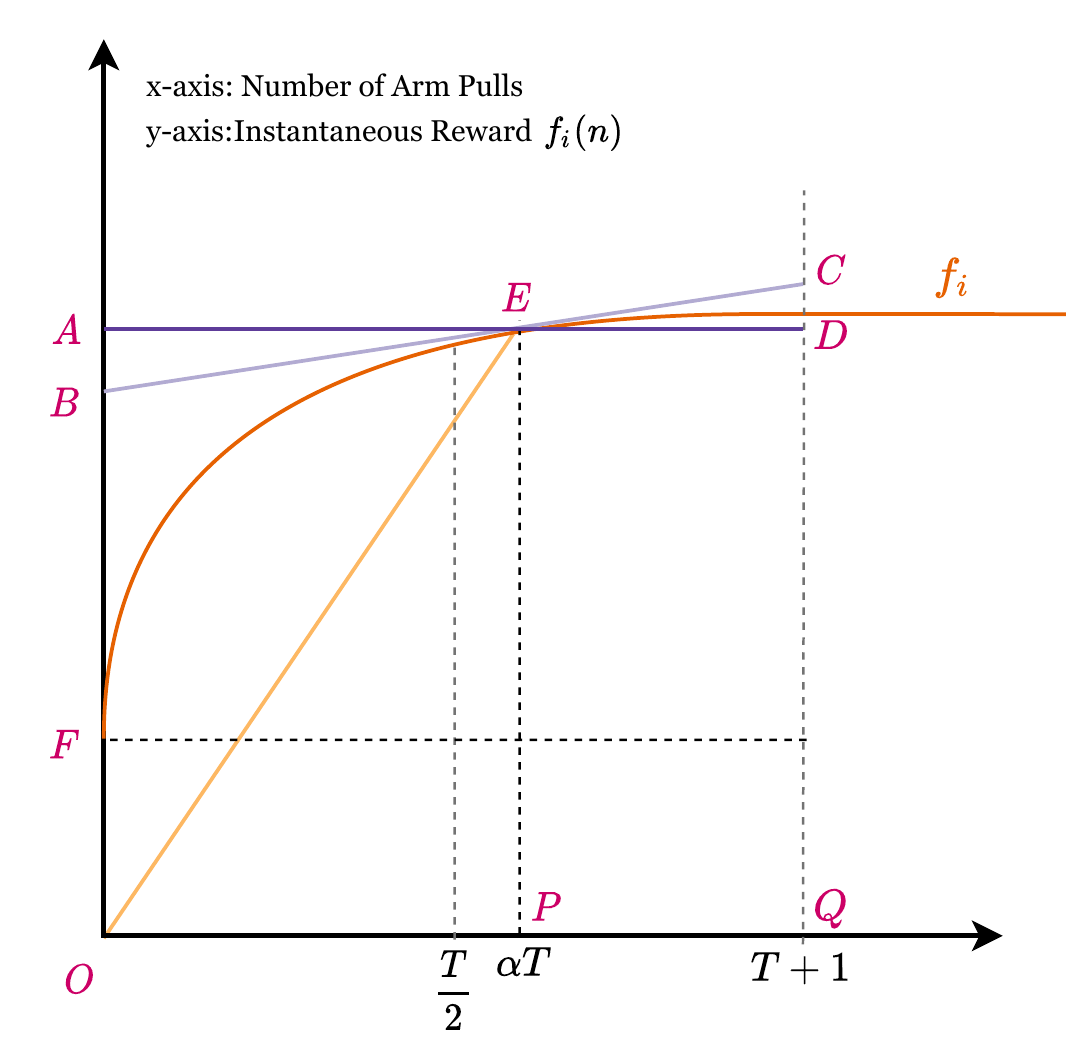}
    \captionof{figure}{$\alpha \geq 1/2$}
    \label{fig: Lemma Case A app}
    \end{minipage}%
    \begin{minipage}{0.5\textwidth}
    \centering
    \includegraphics[scale=0.53]{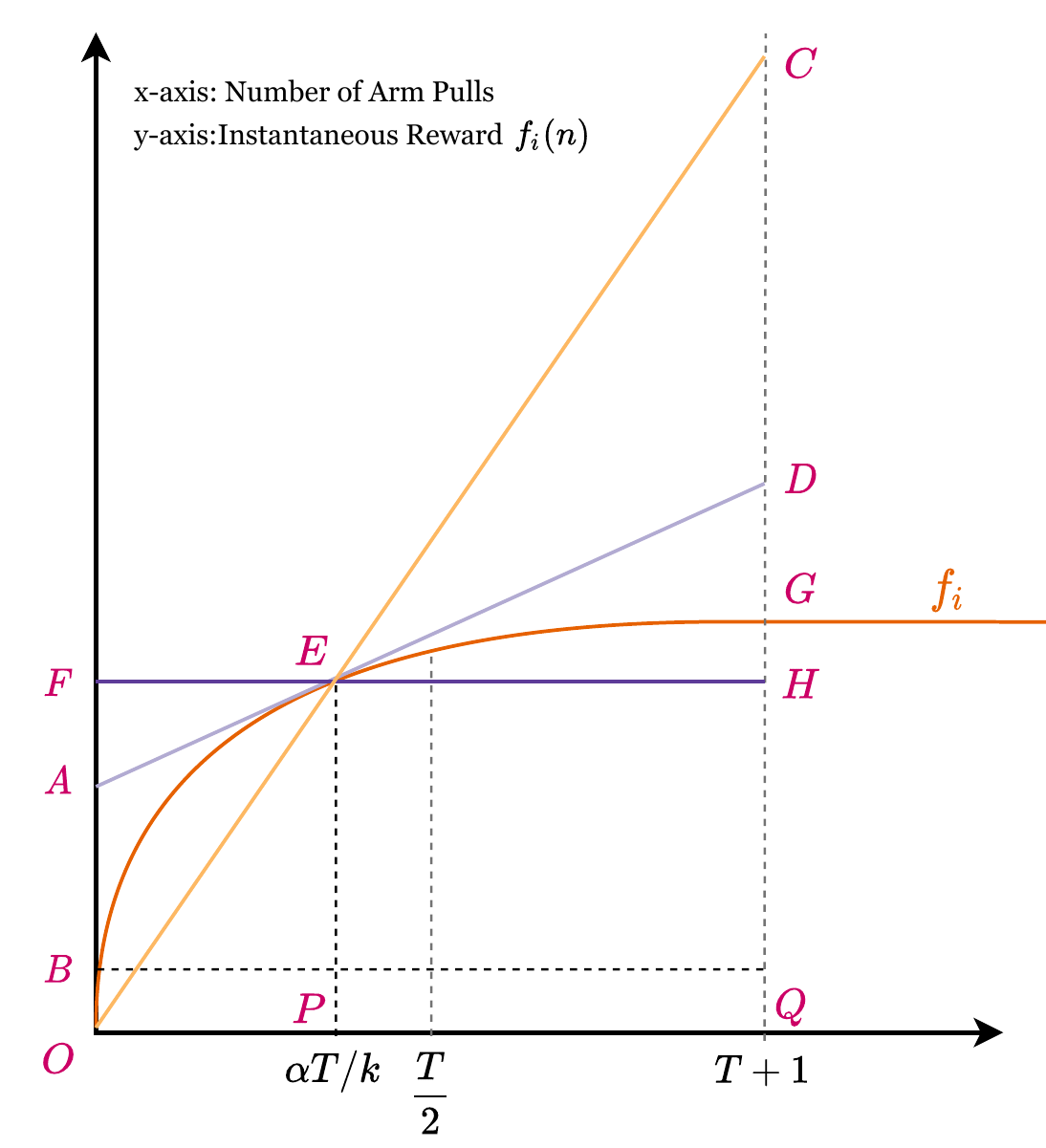}
    \captionof{figure}{$0 < \alpha \leq k/2$}
    \label{fig: Lemma Case B app}
    \end{minipage}
\end{figure}
%%%%%%%%%%%%%%%%End Figures%%%%%%%%%%%%%

\FirstArmToCrossNPulls*
\begin{proof}
Let arm $i$ be pulled $N$ times so far and let $t$ be the time step at which it is pulled for the $N+1$-th time.
Then, we know that
\begin{equation}
\label{eq: opt estimate of first arm to cross N}
    p_i(t) \geq p_j(t) \text{, for all arms $j \neq i$}
\end{equation}
By assumption, arm $i$ is the first arm to be pulled $N+1$ times. This implies that, at $t$, no arm has been pulled more than $N$ times. Hence, $p_i(t) = \rew_i(N)$.
This, along with Equation \ref{eq: opt estimate of first arm to cross N}, gives us
\begin{equation}
\label{eq: opt estimate compare}
    \rew_i(N) \geq p_j(t) \  \text{for all arms $j \neq i$.}
\end{equation}
Now, suppose for contradiction that $\rew_i(N) < \opt(I,N)$, and let $j^*$ denote the optimal arm for $N$ time steps, i.e., $\rew_{j^*}(N) = \opt(I,N)$ (the existence of such an arm $j^*$ follows from Proposition \ref{prop: single arm optimal policy}). %\gan{Nitpicking: For completeness refer to proposition 1 here. Also, for clarity we can say $j^* \neq i$}
Then, 
\begin{align*}
    \opt(I,N) & > \rew_i(N) \tag{From assumption}\\
    & \geq p_{j^*}(t) \tag{From Equation \ref{eq: opt estimate compare}}\\
    & \geq \rew_{j^*}(N) \tag{From the definition of $p_{j^*}(t)$}
\end{align*}
The last inequality in the above equation follows from the definition of $p_{j^*}(t)$ and the decreasing marginal returns property of the reward functions.
This implies that $\rew_{j^*}(N) < \opt(I,N)$, which is a contradiction.
Hence, $\rew_i(N) \geq \opt(I,N)$. Also, by definition of $\opt(I,N)$, we have $\rew_i(N) \leq \opt(I,N)$.
Hence, we get $\rew_i(N) = \opt(I,N)$. 
%\gan{I think the notation for offline OPT is changed...this change is not updated here }
\end{proof}

\RewNOverRewTForArmi*
\begin{proof}

We first prove part (a). Fix an arm $i\in[k]$ with reward function $f_i(\cdot)$ (see Figure \ref{fig: Lemma Case A app}). 
Then,
\begin{align}
    \rew_i(\alpha T) & \geq \Ar(OFEPO)  \nonumber\\
    %& = \text{Area under orange curve between points $F$ and $E$} \nonumber\\
    & \geq \Ar(OEPO)  \nonumber\\
    & = \dfrac{\alpha T}{2}f_i(\alpha T) \label{eq app: alpha greater than half, reward alpha T lower bound}
\end{align}
In the above equation, $\Ar(OFEPO)$ denotes the area under $f_i$ from $0$ to $\alpha T$, and $\Ar(OEPO)$ denotes the area under the triangle defined $OEP$. 
We now use the following claim, to complete the proof of part (a).% proof of which can be found at the end of the proof of this lemma.
\begin{restatable}{claim}{RewTUpperBoundAlphaGreaterThanHalf}
\label{claim: alpha greater than half, reward in T upper bound}
For any arm $i\in [k]$ and $\alpha \geq 1/2$, $\rew_i(T) \leq \dfrac{5T}{4}f_i(\alpha T)$.
\end{restatable}
First, we complete the proof of part (a) using Claim \ref{claim: alpha greater than half, reward in T upper bound}, and then give the proof of the claim. From Claim \ref{claim: alpha greater than half, reward in T upper bound} and Equation \ref{eq app: alpha greater than half, reward alpha T lower bound}, we obtain
\begin{align*}
    \dfrac{\rew_i(\alpha T)}{\rew_i(T)} &\geq \dfrac{\dfrac{\alpha T}{2}f_i(\alpha T)}{\dfrac{5T}{4}f_i(\alpha T)}
     = \dfrac{2\alpha}{5} 
 \geq \dfrac{1}{5} \tag{Since $\alpha \geq 1/2$}
\end{align*}
\begin{proof}[Proof of Claim \ref{claim: alpha greater than half, reward in T upper bound}]
Note that, $\rew_i(T) \leq  Ar(OBCQO)$, where $\Ar(OBCQO)$ is the area under line segment $BC$. 
Further, $\Ar(OADQO) = f_i(\alpha T)(T+1)$, where $\Ar(OADQO)$ is the area under the line segment $AD$. 
%$$\text{Area under line segment } AD = f_i(\alpha T)T.$$
Further, note that $f_i(\alpha T)(T+1) \leq \frac{5T}{4}f_i(\alpha T)$ for $T \geq 4$.
Hence, to prove the claim, it is sufficient to prove that $Ar(OBCQO) \leq \Ar(OADQO)$.
%$$\text{Area under line segment } BC \leq \text{Area under line segment } AD.$$
To show this, we first argue that triangles $ABE$ and $DCE$ are similar. 
Note that,
\begin{align*}
    \angle AEB &= \angle DEC \tag{Vertically Opposite Angles}\\
    \angle ABE &= \angle DCE \tag{Alternate Interior Angles}\\
\end{align*}
This implies triangle $ABE$ is similar to triangle $DCE$. %$$\triangle ABE \sim \triangle DCE.$$
%In the proof of this claim,  
Since $1/2 \leq \alpha \leq 1$, we get
\begin{equation}
\label{eq app: length of AE and ED}
    AE = \alpha (T+1) \geq (1-\alpha) (T+1) = DE
\end{equation}
Since triangle $ABE$ is similar to triangle $DCE$, we have that
$$\dfrac{AE}{DE} = \dfrac{BE}{CE} = \dfrac{AB}{DC}.$$
Along with Equation \ref{eq app: length of AE and ED}, this implies that $BE \geq CE \text{ and } AB \geq DC.$
Hence, $\Ar(ABEA) \geq \Ar(DCED)$, where $\Ar(ABEA)$ and $\Ar(DCED)$ are the areas of the triangles $ABE$ and $DCE$ repsectively.
Therefore, $\Ar(OBCQO) \leq \Ar(OADQO)$.
%$\text{Area}(\triangle ABE) > \text{Area} (\triangle DCE)$, and therefore, $$\text{Area under line segment } BC \leq \text{Area under line segment } AD.$$
\end{proof}
%\RewTUpperBoundAlphaGreaterThanHalf*

Next, we prove part (b) of the lemma. Fix an arm $i \in [k]$ with reward function $f_i(\cdot)$ (see Figure \ref{fig: Lemma Case B app}). Let $\slopeOC$ denote the slope of line segment $OE$.\footnote{The slope of a line passing through two points $(x_1,y_1)$ and $(x_2,y_2)$ is given by $\frac{y_2-y_1}{x_2-x_1}$.} % \gan{maybe we can show it in the figure if it isnt cluttering too much?}. 
Then,
\begin{align}
    \rew_i(\alpha T/k) &\geq \Ar(OEPO) \nonumber\\
    & = \dfrac{1}{2}\cdot \dfrac{\alpha T}{k} \cdot \Big( \dfrac{\alpha T}{k} \slopeOC \Big) \nonumber\\
    & = \dfrac{\alpha^2 T^2 \slopeOC}{2k^2} \label{eq app: alpha less than k by 2, reward alpha T by k lower bound}
\end{align}
In the above equation, $\Ar(OEPO)$ denotes the area under the triangle $OEP$. 
Next, we use the following claim to proof of part (b) of the lemma.
\begin{restatable}{claim}{RewTUpperBoundAlphaGreaterThanKByTwo}
\label{claim: alpha less than k by 2, rew T upper bound}
For any arm $i\in [k]$ and $0 < \alpha \leq k/2$, $\rew_i(T) \leq \dfrac{25T^2}{32} \slopeOC$.
\end{restatable}
%\RewTUpperBoundAlphaGreaterThanKByTwo*
%\gan{I think claims 3 and 4 should be inline. Specially because we are in appendix and there is nothing much to the proof of the lemma than these two claims. }
First, we complete the proof of part (b) using the above claim, and then give the proof of the claim. 
From Claim \ref{claim: alpha less than k by 2, rew T upper bound} and Equation \ref{eq app: alpha less than k by 2, reward alpha T by k lower bound}, we have
\begin{align*}
    \dfrac{\rew_i(\alpha T/k)}{\rew_i(T)} \geq \dfrac{\alpha^2 T^2 \slopeOC}{2k^2} \cdot \dfrac{32}{25T^2 \slopeOC}
     = \dfrac{16\alpha^2}{25k^2}
\end{align*}

%\RewTUpperBoundAlphaGreaterThanHalf*
%\RewTUpperBoundAlphaGreaterThanKByTwo*
\begin{proof}[Proof of Claim \ref{claim: alpha less than k by 2, rew T upper bound}]
%In Figure \ref{fig: Lemma Case B app}, we slightly abuse notation and use $T$ on the x-axis to denote the time horizon and also as a label for the vertex. 
Observe in Figure \ref{fig: Lemma Case B app} that, 
$\rew_i(T) \leq \Ar(OADQO)$, where $\Ar(OADQO)$ is the area of the polygon $OADQ$.
Further, $\Ar(OCQO) \leq \dfrac{(T+1)^2\slopeOC}{2} \leq \dfrac{25T^2\slopeOC}{32}$ ($T \geq 4 \implies T+1 \leq 5T/4$), where $\Ar(OCQO)$ is the area of the triangle $OCQ$.
Hence, to prove the claim, it is sufficient to show that 
\begin{equation}
    \label{eq app: area of quad and traingle}
    \Ar(OADQO) \leq \Ar(OCQO)
\end{equation}
Further, to prove Equation \ref{eq app: area of quad and traingle}, it is sufficient to prove that
\begin{equation}
\label{eq app: area of quad and triangle 2}
    \Ar(AOEA) \leq \Ar(DCED)
\end{equation}
where, $\Ar(AOEA)$ and $\Ar(DCED)$ denotes the areas of triangles $AOE$ and $DCE$ respectively. 
We use the following two observations to prove this.
\begin{observation}
\label{obs: obs 1 similar triangle}
Triangle $AOE$ is similar to triangle $DCE$.
\end{observation}
\begin{observation}
\label{obs: obs 2 similar triangle}
Triangle $EFO$ is similar to triangle $EHC$.
\end{observation}
Observation \ref{obs: obs 1 similar triangle} follows from
\begin{align*}
    \angle AEO & = \angle DEC \tag{Vertically opposite angles}\\
    \angle AOE & = \angle DCE \tag{Alternate interior angles}
\end{align*}
Similarly, Observation \ref{obs: obs 2 similar triangle} follows from 
\begin{align*}
    \angle EFO & = \angle EHC \tag{Both equal to $90\degree$}\\
    \angle FEO & = \angle HEC \tag{Vertically opposite angles}
\end{align*}
%\gan{I think we can void above two proofs. They are evident from figure, no?  }

We now prove the inequality in Eq. \ref{eq app: area of quad and triangle 2}.
First, note that $\alpha \leq k/2 \implies \alpha T/k \leq T/2$, and hence
\begin{equation}
\label{eq app: obs3 eq1}
    FE \leq HE
\end{equation}
From Observation \ref{obs: obs 2 similar triangle}, we have
\begin{equation}
\label{eq app: obs3 eq2}
    \dfrac{FE}{HE} = \dfrac{OE}{CE} = \dfrac{FO}{HC}
\end{equation}
From Equations \ref{eq app: obs3 eq1} and \ref{eq app: obs3 eq2}, we get
\begin{align}
    OE &\leq CE \label{eq app: obs3 eq5}\\
    FO &\leq HC \nonumber
\end{align}
%Similarly, from Eq. \ref{eq app: obs3 eq5}, we know that
%\begin{equation}
%\label{eq: obs3 eq3}
%    OE \leq CE
%\end{equation}
Next, from Observation \ref{obs: obs 1 similar triangle}, we have
\begin{equation}
\label{eq: obs3 eq4}
    \dfrac{OE}{CE} = \dfrac{AE}{DE} = \dfrac{AO}{DC}
\end{equation}
From Equations \ref{eq app: obs3 eq5} and \ref{eq: obs3 eq4}, we get
\begin{align}
    AE &\leq DE \label{eq: obs3 eq6}\\
    AO &\leq DC \label{eq: obs3 eq7}
\end{align}
Finally, from Equations \ref{eq app: obs3 eq5}, \ref{eq: obs3 eq6}, and \ref{eq: obs3 eq7}, we get
\begin{equation*}
    \Ar(AOEA) \leq \Ar(DCED)
    %\text{Area} (\triangle AOE) \leq \text{Area} (\triangle DCE)
\end{equation*}
This concludes the proof of Claim \ref{claim: alpha less than k by 2, rew T upper bound}.
\end{proof}

This concludes the proof of Lemma \ref{lem: rew N over rew T for arm i}.
\end{proof}

\OptNoverOptTforArmi*
%\begin{corollary}
%\label{corr app: opt N by opt T for arm i}
%For any finite time horizon $T$, we have
%\begin{enumerate}[label=(\alph*)]
%    \item $\dfrac{\opt(\alpha T)}{\opt(T)} \geq \dfrac{1}{4}$ for $\dfrac{1}{2} < \alpha \leq 1$
%    \item $\dfrac{\opt(\alpha T/k)}{\opt(T)} \geq \dfrac{\alpha^2}{k^2}$ for $0 < \alpha \leq \dfrac{k}{2}$
%\end{enumerate}
%\end{corollary}
\begin{proof}
\textit{Part (a): }
Fix a time horizon $T$ and let $\alpha \in [1/2,1]$.
By definition of $\opt(I,T)$, we know that $\opt(I,\alpha T) \geq \rew_i(\alpha T)$ for all arms $i\in [k]$.
Let arm $i^* \in [k]$ be the optimal arm for horizon $T$.
Then, $\rew_{i^*}(T) = \opt(I,T)$.
Hence, we have
\begin{align*}
    \dfrac{\opt(I,\alpha T)}{\opt(I,T)} \geq \dfrac{\rew_{i^*}(\alpha T)}{\rew_{i^*}(T)} 
     \geq \dfrac{1}{5} \tag{From Lemma \ref{lem: rew N over rew T for arm i}a}
\end{align*}

\textit{Part (b): }
As before, 
Fix a time horizon $T$ and let $\alpha \in (0,k/2]$.
By definition of $\opt(I,T)$, we know that $\opt(I,\alpha T/k) \geq \rew_i(\alpha T/k)$ for all arms $i\in [k]$.
Let arm $i^* \in [k]$ be the optimal arm for horizon $T$.
Then, $\rew_{i^*}(T) = \opt(I,T)$.
Hence, we have
\begin{align*}
    \dfrac{\opt(I,\alpha T/k)}{\opt(I,T)}  \geq \dfrac{\rew_{i^*}(\alpha T/k)}{\rew_{i^*}(T)} 
     \geq \dfrac{16\alpha^2}{25k^2} \tag{From Lemma \ref{lem: rew N over rew T for arm i}b}
\end{align*}

\end{proof}
%
%\gan{Should we have this statement in the statement of lemma for completeness? }
\RewNOverOPTTForArmi*
\begin{proof}
We prove the lemma separately for the following two cases:
\begin{enumerate}
    \item Arm $i$ is the first arm to cross $N_i - 1$ pulls
    \item Arm $i$ is not the first arm to cross $N_i - 1$ pulls
\end{enumerate}
\textbf{Case 1}: Arm $i$ is the first arm to cross $N_i - 1$ pulls. Since arm $i$ is the first arm to cross $N_i-1$ pulls, it is the first arm to be pulled for $N_i$-th time. Hence, we have
\begin{align*}
    \rew_i(N_i) &\geq \rew_i(N_i - 1) \tag{$f_i$'s are non-negative implies $\rew_i$ is a non-decreasing function}\\
    & = \opt(I, N_i - 1) \tag{From Lemma \ref{lemms: first arm to cross N pulls}}\\
    & \geq \opt(I, N_i/2) 
\end{align*}
The last inequality follows from the fact that $N_i \geq 2$ implies $N_i-1\geq N_i/2$ and $f_i$'s are non-negative implies that $\opt(I,T)$ is a non-decreasing function. From this, we obtain
\begin{equation}
    \label{eq: case 1 rew N over opt T}
    \dfrac{\rew_i(N_i)}{\opt(I,T)} \geq \dfrac{\opt(I,N_i/2)}{\opt(I,T)}
\end{equation}
Now, let $\dfrac{N_i}{2} = \dfrac{\alpha T}{k}$. Since, $N_i \leq N_1 \leq T/2$, we have $\alpha \leq k/2.$ Then,
\begin{align*}
    \dfrac{\rew_i(N_i)}{\opt(I,T)} &\geq \dfrac{\opt(I,N_i/2)}{\opt(I,T)} \tag{From Equation \ref{eq: case 1 rew N over opt T}}\\
    & = \dfrac{\opt(\alpha T/k)}{\opt(I,T)} \tag{Substituting $N_i = \alpha T/k$}\\
    & \geq \dfrac{\alpha^2}{4k^2} \tag{From Corollary \ref{corr: opt N by opt T for arm i}b}\\
    & = \dfrac{N_i^2}{16T^2} \tag{Substituting $\alpha = \dfrac{kN_i}{2T}$}\\
    & \geq \dfrac{N_i^2}{32T^2}
\end{align*}
%\gan{nitpicking: last inequality is strict...}
%
\textbf{Case 2}: Arm $i$ is not the first arm to cross $N_i - 1$ pulls. Let $t$ be the time step at which arm $i$ is pulled for the $N_i$-th time by our algorithm. At $t$, let $N = \max_{j\in[k]} N_j(t)$, i.e., $N$ is the maximum number of pulls any arm has received until time step $t$. Further, among all arms that have been pulled $N$ number of times at time step $t$, let $\ell \in [k]$ denote the arm index of the first arm to cross $N-1$ pulls, i.e., arm $\ell$ is the first arm to be pulled for the $N$-th time. Next, define $N'_\ell = N - 1$. In the remainder of the proof, we work with $N'_\ell$ instead of $N$.

%For ease of exposition, we work with $N'_\ell$ instead of $N$, where $N'_\ell = N - 1$, or equivalently $N = N'_\ell + 1$.

Let $a = N_i -1$ and $b = N'_\ell + 1 - (N_i - 1) = N'_\ell -N_i + 2$. Further, let $\Delta_i = \Delta_i(N_i - 1) = f_i(N_i-1) - f_i(N_i-2)$.
From Lemma \ref{lemms: first arm to cross N pulls}, we have $\rew_\ell(N'_\ell) = \opt_\ell(I,N'_\ell)$. Further, we know that Algorithm \ref{algo: horizon unaware improving bandits} pulled arm $i$ at time step $t$. This implies
\begin{align}
    \rew_i(a) + 2bf_i(a) + 2b^2\Delta_i & \geq \rew_i(a) + (b+1)f_i(a) + \frac{1}{2} (b+1)^2 \Delta_i \tag{Since $b \geq 1$}\\
    & \geq p_i(t) \tag{Using left Riemann sum}\\
    %&\geq \rew_i(a) + \frac{1}{2}b(2f_i(a) + \Delta_i b) \tag{RHS is a continuous approximation of $p_i(t)$}\\
    &\geq \rew_\ell(N'_\ell + 1) \nonumber \tag{Since this equal $p_\ell(t)$ and arm $i$ was pulled}\\
    & \geq \rew_\ell(N'_\ell) \nonumber\\
    & = \opt(I,N'_\ell) \label{eq: optimistic estimate inequality}
\end{align}

\begin{figure}
\centering
    \includegraphics[scale=0.6]{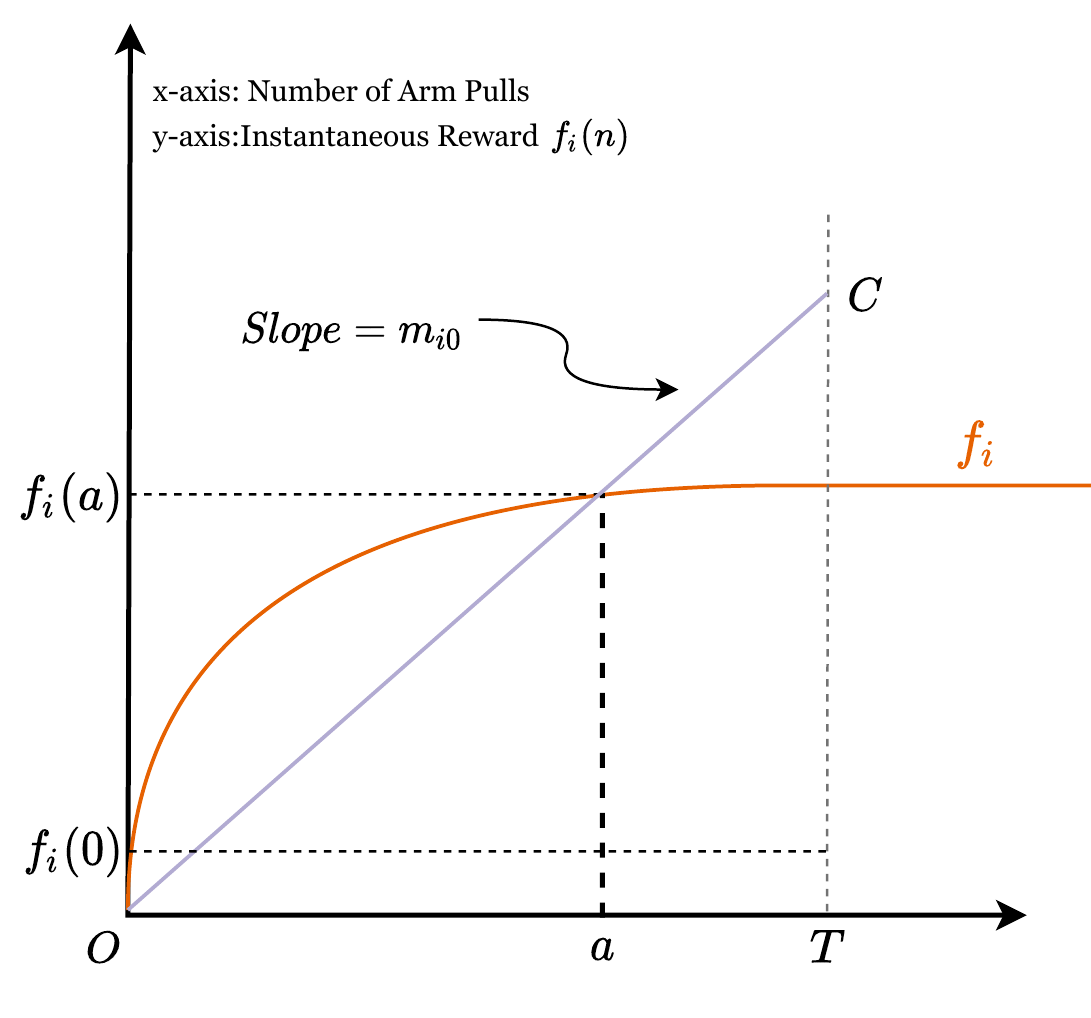}
    \captionof{figure}{Illustration for Case 2 of Lemma \ref{lemma: rew N by opt T for arm i} }
    \label{fig: Lemma Proof app}
\end{figure}

We now consider two sub-cases of Case 2.

\noindent\textbf{Case 2a:} {$f_i(a) \geq b\Delta_i$}. In this case, from Equation \ref{eq: optimistic estimate inequality} we obtain
\begin{equation*}
   \rew_i(a) + 4bf_i(a) \geq \rew_i(a) + 2b(f_i(a) + b\Delta_i) \geq \opt(I,N'_\ell)
\end{equation*}
which implies that,
\begin{equation}
\label{eq: case 2a}
    f_i(a) \geq \dfrac{\opt(I,N'_\ell) - \rew_i(a)}{4b}
\end{equation}
From the non-decreasing and decreasing marginal returns property of $f_i$, we get
\begin{align*}
    \rew_i(a) &\geq \dfrac{a}{2} f_i(a) \tag{From the properties of $f_i$}\\
    & \geq \dfrac{a}{2} \dfrac{\opt(I, N'_\ell) - \rew_i(a)}{4b} \tag{From Eq. \ref{eq: case 2a}}\\
    \implies \rew_i(a)\Big(1 + \dfrac{a}{8b}\Big) &\geq \dfrac{a}{8b} \opt(I, N'_\ell)\\
    \implies \rew_i(N_i) \geq \rew_i(a) &\geq \dfrac{a}{a+8b} \opt(I, N'_\ell)\\
    &\geq \dfrac{a}{8(a+b)} \opt(I, N'_\ell)\\
    &\geq \dfrac{N_i}{32 N'_\ell}\opt(I, N'_\ell)
\end{align*}
The last inequality is obtained as follows: $a = N_i - 1 \geq N_i/2$ for $N_i \geq 2$. Further, $a+b = N'_\ell + 1 \leq 2N'_\ell$ for $N'_\ell \geq 1$. Next, dividing by $\opt(I, T)$ on both sides,
\begin{equation}
\label{eq: rew over opt case 2a}
    \dfrac{\rew_i(N_i)}{\opt(I, T)} \geq \dfrac{N_i}{32N'_\ell}\cdot\dfrac{\opt(I, N'_\ell)}{\opt(I, T)}
\end{equation}

Let $N'_\ell = \alpha T/k$. Since $N_1 \leq T/2$, we know $\alpha \leq k/2$. Then,
\begin{align*}
    \dfrac{\rew_i(N_i)}{\opt(I, T)} & \geq \dfrac{kN_i}{32\alpha T} \cdot \dfrac{16\alpha^2}{25k^2} \tag{Substituting for $N'_\ell$ and from Cor. \ref{corr: opt N by opt T for arm i}b}\\
    & = \dfrac{\alpha N_i}{50kT}\\
    & = \dfrac{N_i}{50T^2}\cdot \dfrac{\alpha T}{k}\\
    & = \dfrac{N_i}{50T^2}\cdot N'_\ell\\
    & \geq \dfrac{N_i}{50T^2}\cdot\dfrac{N_i}{2} \tag{$N'_\ell + 1 \geq N_i \implies N'_\ell \geq N_i/2$ for $N_i\geq 2$}\\
    & = \dfrac{N_i^2}{100T^2} 
\end{align*}

\noindent\textbf{Case 2b}: $f_i(a) < b\Delta_i$.
In this case, from Equation \ref{eq: optimistic estimate inequality} we obtain
\begin{equation*}
    \rew_i(a) + 4b^2\Delta_i \geq \rew_i(a) + 2b(f_i(a) + b\Delta_i) \geq \opt(I, N'_\ell)
\end{equation*}
which implies that,
\begin{equation}
\label{eq: case 2b}
    \Delta_i \geq \dfrac{\opt(I, N'_\ell) - \rew_i(a)}{4b^2}
\end{equation}
Let $\slopeIO = \dfrac{f_i(a) - 0}{a - 0} = \dfrac{f_i(a)}{a}$. From the decreasing marginal returns property of $f_i$, we get $\slopeIO \geq \Delta_i.$
Further,
\begin{align}
    \rew_i(a) &\geq \dfrac{1}{2}a f_i(a) \tag{See Figure \ref{fig: Lemma Proof app}} \nonumber\\
    & = \dfrac{1}{2}a^2 \slopeIO \tag{From the definition of $\slopeIO$ above} \nonumber\\
    & \geq \dfrac{1}{2}a^2 \Delta_i  \nonumber\\
    & \geq \dfrac{1}{2}\cdot \dfrac{a^2}{4b^2} [\opt(I, N'_\ell) - \rew_i(a)] \tag{From Eq. \ref{eq: case 2b}} \nonumber\\
    \implies \Big(1 + \dfrac{a^2}{8b^2}\Big) \cdot \rew_i(a) &\geq \dfrac{a^2}{8b^2}\opt(I, N'_\ell) \label{eq: case 2b rew N lowerB}
\end{align}
Then, we get
\begin{align*}
    \rew_i(N_i) &\geq \rew_i(a) \tag{$N_i > a = N_i-1$ and $\rew_i$ is non-decreasing}\\
    &\geq \dfrac{a^2}{a^2+8b^2}\opt(I, N'_\ell) \tag{From Equation \ref{eq: case 2b rew N lowerB}}\\
    &\geq \dfrac{a^2}{8(a^2 + b^2)}\opt(I, N'_\ell)\\
    & \geq \dfrac{N_i^2}{128{N'_\ell}^2}\opt(I, N'_\ell)
\end{align*}
The last inequality is obtained as follows: $a = N_i - 1 \geq N_i/2$ for $N_i \geq 2$ which implies $(N_i - 1)^2 \geq N_i^2/4$ for $N_i \geq 2$. Further, $a^2 + b^2 \leq (a+b)^2 = (N'_\ell + 1)^2 \leq 4{N'_\ell}^2$ for $N'_\ell \geq 1$. Dividing by $\opt(I,T)$ on both sides,
\begin{equation}
\label{eq: rew over opt case 2b}
    \dfrac{\rew_i(N_i)}{\opt(I, T)} \geq \dfrac{N_i^2}{128{N'_\ell}^2}\cdot\dfrac{\opt(I, N'_\ell)}{\opt(I, T)}
\end{equation}
Let $N'_\ell = \dfrac{\alpha T}{k}$. As before, $N_1 \leq T/2 \implies \alpha \leq k/2.$
\begin{align*}
    \dfrac{\rew_i(N_i)}{\opt(I, T)} &\geq \dfrac{k^2N_i^2}{128\alpha^2T^2}\cdot \dfrac{16\alpha^2}{25k^2}\tag{Substituting for $N'_\ell$ and from Lemma \ref{lem: rew N over rew T for arm i}b}\\
    &= \dfrac{N_i^2}{200T^2}.
\end{align*}

\end{proof}

%\OptAlgoCompetitiveRatio*
Now we complete the proof of Theorem \ref{theorem: k competitive algo proof} using the above lemmas and corollary.
\begin{proof}[Proof of Theorem \ref{theorem: k competitive algo proof}]
We look at the following two cases:
\begin{enumerate}
    \item $N_1 > T/2$
    \item $N_1 \leq T/2$
\end{enumerate}
\noindent\textbf{Case 1: } $N_1 > T/2$ 

$N_1 > T/2$ implies that arm $1$ is the only arm to cross $T/2$ pulls, and hence, the first arm to cross $T/2$ pulls. From Lemma \ref{lemms: first arm to cross N pulls}, we get $$\rew_1(T/2) = \opt(I, T/2).$$
From Corollary \ref{corr: opt N by opt T for arm i}(a), we get $$\dfrac{\opt(I, T/2)}{\opt(I, T)} \geq \dfrac{1}{5}.$$
Further, we know that $\alg(I, T) \geq \rew_1(T/2) = \opt(I, T/2)$.
From this, we get $$\dfrac{\opt(I, T)}{\alg(I, T)} \leq \dfrac{\opt(I, T)}{\opt(I, T/2)} \leq 5 \leq 200k.$$

\noindent\textbf{Case 2: } $N_1 \leq T/2$ \\
From Lemma \ref{lemma: rew N by opt T for arm i}, we have $$\dfrac{\rew_i(N_i)}{\opt(I,T)} \geq \dfrac{N_i^2}{200T^2}.$$
Further, $\alg(I,T) = \sum_{i\in [k]}\rew_i(N_i)$.
From this we get 
\begin{align*}
    \dfrac{\alg(I,T)}{\opt(I,T)} &= \dfrac{\sum_{i\in [k]}\rew_i(N_i)}{\opt(I,T)}
     = \sum_{i\in [k]} \dfrac{\rew_i(N_i)}{\opt(I,T)}\\
    & \geq \sum_{i\in[k]} \dfrac{N_i^2}{200T^2}\\
    & \geq \dfrac{1}{200k} \tag{From Observation \ref{obs: cauchy schwarz}}
\end{align*}
The last inequality in the above equation follows from Observation \ref{obs: cauchy schwarz} stated next.
\begin{observation}
\label{obs: cauchy schwarz}
Let $N_1, N_2, \ldots, N_k \in \mathbb{N}$ such that $N_1 + N_2 + \ldots + N_k = T$. Then $\sum_{i\in[k]} N_i^2 \geq \dfrac{T^2}{k}$.
\end{observation}
\begin{proof}
Let $\mathbf{u} = (N_1, N_2,\ldots,N_k)$ and $\mathbf{v} = (1,1,\ldots,1)$. Then,
\begin{align*}
    \lVert \mathbf{u} \rVert^2 \cdot \lVert \mathbf{v} \rVert^2 & \geq |\langle \mathbf{u},\mathbf{v} \rangle|^2 \tag{From Cauchy-Schwarz Inequality}\\
    \implies \sum_{i\in[k]} N_i^2 \cdot k & \geq T^2 \tag{From the definition of $\mathbf{u}$ and $\mathbf{v}$}\\
    \implies \sum_{i\in[k]} N_i^2 &\geq \dfrac{T^2}{k}
\end{align*}
This completes the proof of Observation \ref{obs: cauchy schwarz}.
\end{proof}
From the arguments above it follows that $\dfrac{\opt(I,T)}{\alg(I,T)} \leq 200k$. Since this holds for an arbitrary instance $I$, we have that $\compRatio(\alg) \leq 200k$.
\end{proof}

\section{Proof of Theorem \ref{theorem: fairness of optimal algo}}
\label{app: proof of fairness}

%-------------------------------------------Fairness-----------------------------------------------------
\OptAlgoFairness*
\begin{proof}
The proof of the theorem is completed using Lemma \ref{lemma: fairness of optimal algo} stated after the proof of the theorem. 
Suppose $L_i$ as defined in Lemma \ref{lemma: fairness of optimal algo} is finite for an arm $i\in [k]$. 
Then from Lemma \ref{lemma: fairness of optimal algo} we have $\Delta(L_i) = 0$ and the marginal decreasing property of the reward functions ensures that arm $i$ has reached its true potential,
i.e, $f_i(L_i) = a_i$. Further, if $L_i$ is not finite then again from Lemma \ref{lemma: fairness of optimal algo} the arm is pulled infinitely many times. Hence,
from the properties of the reward functions we have that for every $\varepsilon \in (0,a_i]$, there exists $T\in \mathbb{N}$ such that $\alg$ ensures the following:
$a_i - f_i(N_{i}(T)) \leq \varepsilon \,.$
\end{proof}

\FairnessLemma*
%\gan{ I think it should be $L_i = \arg\max_{T \in \mathbb{N}}\{N_i(T)\}$ instead!!}
\begin{proof}
First, we note that there exists an arm $i \in [k]$ such that $L_i$ is infinite. 
If not, then the algorithm will not pull any arm beyond $\sum_{i\in[k]} L_i$ time steps which gives a contradiction.
Hence, there exists an arm that is pulled infinitely many times.
Without loss of generality, let this be arm $1$.
Hence, we conclude that the lemma holds for arm $1$.
We prove the lemma by showing that the property in the lemma cannot hold for just $\ell \in [1,k-1]$ arms.
For the sake of contradiction, assume that the lemma holds for exactly $1\leq \ell < k$ arms and does not hold for the remaining $k-\ell$ arms.
Without loss of generality let arms $C = \{1,\ldots,\ell\}$ be the set of arms for which the lemma holds and $\overline{C} = \{\ell+1,\ldots,k\}$ be the set of arms for which the lemma does not hold.
Therefore, for all arms $j\in \overline{C}$, $L_j$ is finite and $\Delta_{j}(L_j) \neq 0$.
%Let $\Delta = \max_{i\in \overline{C}} \Delta_i(L_i)$.
%Without loss of generality assume that $\ell+1 = \arg\max_{i\in \overline{C}} \Delta_i(L_i)$.
For clarity of writing, we use $m$ to denote $\ell+1$, i.e., $m=\ell+1$.
Let $\Delta = \Delta_m(L_m)$.
We will show that there exists an arm $j$ in $\overline{C}$ that will be pulled more than $L_j$ times, thus leading to a contradiction.
%It is sufficient to prove this for any one arm in $\overline{C}$.
%Without loss of generality 
%In particular, we show that there exists a large enough time horizon arm $m$.

Let $T_i = \arg\max_{t\in \mathbb{N}}\{N_{i,t}\}$ for all $i\in [k]$. 
Note that if $L_i$ is finite then $T_i$ is the time step at which arm $i$ is pulled for the $L_i$-th time.
Next, we choose $T\in \mathbb{N}$ such that:
\begin{enumerate}
    \item $T > \max_{j\in \overline{C}} T_j$
    \item $\Delta_i(N_i(T)) < \frac{\Delta}{2}$ for all $i\in C$
\end{enumerate}

Note that $L_j$ if finite for all $i\in \overline{C}$. Hence, it is easy to find $T$ satisfying (1). Further, for any $i\in C$, one of the following holds:
\begin{enumerate}
    \item $L_i$ is infinite: In this case, since the reward functions are bounded in $(0,1)$ and they have decreasing marginal returns,
    we can conclude that such a $T$ exists. %In which case from the bounded-ness (in $[0,1]$) and decreasing marginal returns property of $f_i$'s, we can conclude that such a $T$ exists.
    \item $L_i$ is finite and $\Delta_i(L_i) = 0$: In this case, for some $T \leq L_i$, $\Delta_i(N_i(T)) < \frac{\Delta}{2}$.
\end{enumerate}

To show that there exists an arm $j\in \overline{C}$ that is pulled more than $L_j$ times, we show that there exists $T'\geq T$ such that the optimistic estimate of
arm $m$, $p_m(T')$, is more than the optimistic estimate, $p_i(T')$ of any arm $i\in C$. This would imply at time $T'$ either of the following
choices are made by the algorithm: a) if $p_m(T') > p_j(T')$ for all $j\in \overline{C}\setminus \{m\}$ then arm $m$ is pulled by the algorithm,
or b) if there exists $j\in C$ such that $p_m(T') < p_j(T')$ then an arm from $\overline{C}\setminus \{m\}$ is pulled. Notice that in both cases an arm from
set $\overline{C}$ is pulled leading to a contradiction. Hence, to complete the proof of the lemma,
we show that there exists $T' \geq T$ such that $p_m(T') > p_i(T')$ for all $i\in C$.

Let arm $i$ be some arm in the set $C$. 
At $T'$ let $N^*$ denote the number of arm pulls of the arm that has been pulled the maximum number of times in $T'$ time steps.
We first provide an upper bound on $p_i(T')$.
Note that 
\begin{align}
    p_i(T') & = \rew_i(N_i(T')) + \sum_{n=1}^{N^*-N_i(T')} [f_i(N_i(T')) + n\cdot\Delta_i(N_i(T'))] \nonumber\\
    & \substack{\leq \\ (i)} \rew_i(N_i(T)) + \sum_{n=1}^{N^* - N_i(T)} [f_i(N_i(T)) + n\cdot\Delta_i(N_i(T))] \nonumber\\
    & \leq \rew_i(N_i(T)) + \sum_{n=1}^{N^* - N_i(T)} \Big[f_i(N_i(T)) + n\cdot\frac{\Delta}{2}\Big] \tag{Since $T$ is such that $\Delta_i(N_i(T)) < \Delta/2$ for $i\in C$}
    \nonumber\\
    & = \rew_i(N_i(T)) + (N^* - N_i(T))f_i(N_i(T)) + (N^* - N_i(T))(N^* - N_i(T) + 1)\frac{\Delta}{4} \nonumber\\
    &\leq \rew_i(N_i(T)) + (N^* - N_i(T))f_i(N_i(T)) + (N^* - N_i(T) + 1)^2 \frac{\Delta}{4}  \label{eq: bounding piT'}
\end{align}
Inequality (i) in the above above holds because, $T \leq T'$ implies that $N_i(T) \leq N_i(T')$.
Hence, note that the optimistic estimate computed with $N_i(T)$ is at least as much as that computed with $N_i(T')$ from the concavity and decreasing marginal returns property of $f_i$.
Next observe that, 
$$p_m(T') \geq \rew_m(L_m) + (N^* - L_m)f_m(L_m) + (N^* - L_m)^2\frac{\Delta}{2}.$$
To complete our proof, we need to show that $p_m(T') - p_i(T') > 0$.
%We have shown that $p_i(T') \leq \rew_i(N_i(T)) + (N^* - N_i(T))f_i(N_i(T)) + (N^* - N_i(T) + 1)^2 \frac{\Delta}{2}$.
From Eqn. \ref{eq: bounding piT'}, it is sufficient to show that
$$p_m(T') - \rew_i(N_i(T)) - (N^* - N_i(T))f_i(N_i(T)) - (N^* - N_i(T) + 1)^2 \frac{\Delta}{2} > 0.$$
We now analyze this quantity. Note that,
\begin{align}
    p_m(T') - p_i(T') &\geq \rew_m(L_m) + (N^* - L_m)f_m(L_m) + (N^* - L_m)^2\frac{\Delta}{2} \nonumber\\
    &~~~~ - \rew_i(N_i(T)) - (N^* - N_i(T))f_i(N_i(T)) - (N^* - N_i(T) + 1)^2 \frac{\Delta}{4} \nonumber\\
    & = \rew_m(L_m) - \rew_i(N_i(T)) + (N^* - L_m)f_m(L_m) - (N^* - N_i(T))f_i(N_i(T)) \nonumber\\
    &~~~~ + (N^* - L_m)^2\frac{\Delta}{2} - (N^* - N_i(T) + 1)^2 \frac{\Delta}{4} \label{eq: N large enough}
\end{align}
Note that all other terms on the RHS above, except $N^*$, are constant with respect to $T'$. 
Further, $N^*$ increases as $T'$ increases.
For large enough $N^*$, the ${(N^*)}^2$ terms (Eq. \ref{eq: N large enough}) will dominate.
Since ${(N^*)}^2 (\Delta/2 - \Delta/4) = \frac{{(N^*)}^2\Delta}{4}$ is a positive quantity, we can find $T'$ large enough such that the quantity on the RHS in Eq. \ref{eq: N large enough} is positive.

%From this, we can conclude that every arm in $\overline{C}$ will be pulled again.
%This is a contradiction.
We have shown that the property in the lemma cannot be true for only $\ell\in [1,k-1]$ arms.
Hence, we conclude that the lemma holds for all arms $i\in [k]$. % and hence $\overline{C} = \emptyset$
\end{proof}

\end{document}